\documentclass{article}

\PassOptionsToPackage{numbers, compress}{natbib}
\usepackage[preprint]{neurips_2026}

\usepackage[utf8]{inputenc} 
\usepackage[T1]{fontenc}    
\usepackage{hyperref}       
\usepackage{url}            
\usepackage{booktabs}       
\usepackage{amsfonts}       
\usepackage{nicefrac}       
\usepackage{microtype}      
\usepackage{xcolor}         
\usepackage{amsmath}
\usepackage{amssymb}
\usepackage{mathtools}
\usepackage{amsthm}
\usepackage[normalem]{ulem}
\usepackage{url}
\usepackage{multirow}
\usepackage[capitalize,noabbrev]{cleveref}
\Crefname{equation}{Eq.}{Eqs.}
\Crefname{section}{Sec.}{Secs.}
\Crefname{figure}{Fig.}{Figs.}
\Crefname{table}{Tab.}{Tabs.}

\usepackage{easyReview}
\usepackage{physics}

\usepackage{wrapfig}
\usepackage{enumitem}
\usepackage{subcaption}
\usepackage{algpseudocodex}
\usepackage{pgfplots}
\pgfplotsset{compat=newest}
\usepackage[most]{tcolorbox}
\usepackage{algorithm}
\usepackage{thm-restate}

\usepackage{dirtytalk}
\theoremstyle{plain}
\newtheorem{theorem}{Theorem}
\theoremstyle{definition}
\newtheorem{definition}{Definition}
\theoremstyle{remark}
\newtheorem{remark}{Remark}
\theoremstyle{proposition}
\newtheorem{proposition}{Proposition}

\Crefname{lemma}{Lemma}{Lemmas}
\Crefname{corollary}{Corollary}{Corollaries}
\Crefname{proposition}{Proposition}{Propositions}

\usepackage[dvipsnames]{xcolor}
\title{Multi-agent decision making: A Blackwell's informativeness approach}

\author{
    Zheng Zhang\textsuperscript{\rm 1},
    Cuong C. Nguyen\textsuperscript{\rm 1},
    Kevin Wells\textsuperscript{\rm 1},
    Gustavo Carneiro\textsuperscript{\rm 1}\thanks{Corresponding authors}\\
    \textsuperscript{\rm 1}Centre for Vision, Speech and Signal Processing, University of Surrey\\
    \texttt{zheng.zhang@surrey.ac.uk}
    }

\begin{document}

\maketitle


\begin{abstract}
    The rapid development of large language models (LLMs) has motivated research on decision-making in multi‑agent systems, where multiple agents collaborate to achieve shared objectives. Existing aggregation approaches, such as voting and debate, are largely ad-hoc and lack formal guarantees regarding the informativeness of the resulting decisions. In this paper, we provide a principled approach to analyse decisions made in the multi‑LLM setting using Blackwell's informativeness framework. Within the Blackwell information-structure abstraction, we show that voting and debate induce information structures that are no more informative than the pooled private information of all agents. This result identifies Bayesian pooled posterior maximisation as an information-theoretic upper-bound decision rule under the Blackwell ordering. Motivated by this theoretical analysis, we introduce a practical method for LLM-based question-answering (QA) tasks that estimates each agent's posterior and approximates the pooled posterior using a  product-of-posteriors estimator. Extensive experiments on six QA benchmarks demonstrate that our approach outperforms state‑of‑the‑art multi-LLM debate and voting methods.
\end{abstract}

\section{Introduction}
\label{sec:introduction}
    Recent advances in large language models (LLMs) have enabled the development of autonomous intelligent agents 
    that can interact
    with their environment, plan
    actions, and make
    decisions with minimal human intervention~\citep{yao2022react}. As individual LLM-based agents~\citep{guo2024large} exhibit domain-specific strengths, such as proficiency in programming~\citep{li2022competition}, biological analysis~\citep{singhal2023large}, or mathematical reasoning~\citep{lewkowycz2022solving}, \emph{multi-agent decision-making} frameworks have emerged to exploit the complementary capabilities of multiple specialised agents~\citep{du2023improving,li202024prd,guo2024large}. 
    In principle, these frameworks coordinate agents' actions and beliefs to integrate heterogeneous expertise and produce more informed and reliable decisions.

    In the context of LLMs, multi-agent decision making is commonly realised through 
    voting and debate~\citep{choi2025debate}.
    \emph{Voting}, particularly majority voting, aggregates 
    agents' independent judgments into a collective decision~\citep{wang2023self}. 
    While computationally efficient and simple, it disregards variation in agent expertise, 
    which may reduce informativeness and lead to suboptimal decisions.
    In contrast, \emph{multi-agent debate} (MAD)~\citep{chan2024chateval} enables agents to iteratively exchange perspectives and update their beliefs through structured, multi-round discussions, with the goal of reaching a more coherent and higher-quality collective decision.
    Despite this promise, recent empirical studies~\citep{choi2025debate} indicate that multi-round debate is often less reliable and less effective than simple voting-based approaches. 
    MAD systems are shown to frequently converge toward majority-held misconceptions rather than correct reasoning~\citep{estornell2024multi}. In homogeneous settings, straightforward ensemble techniques 
    often outperform multi-round debate~\citep{choi2025debate}.
    MAD is also reported to fail to surpass even single-agent baselines~\citep{zhang2025stop}. 
    More broadly, LLMs exhibit substantial performance degradation in multi-turn settings regardless of protocol~\citep{laban2026llms}, suggesting the limitation is not unique to debate.
    These observations motivate studying when multi-LLM systems fail to extract all available information, regardless of protocol.

    In this paper, we analyse multi-agent LLM decision making through Blackwell's informativeness framework~\citep{blackwell1953equivalent}. This perspective allows us to compare voting, debate, and posterior-level aggregation in terms of the information structures. We show that voting and debate are less informative than the agents' joint private information, which serves as the Blackwell-optimal reference rule. Motivated by this analysis, we develop MA-PoP, a practical method for multiple-choice QA that estimates the pooled posterior directly by combining per-agent posteriors. Our contributions are as follows:
    \begin{itemize}[topsep=0pt, itemsep=0pt, leftmargin=2ex]
        \item We provide a Blackwell-based formalisation of multi-agent LLM aggregation, showing how voting and debating can lose information relative to posterior-level access to agents' joint private evidence.
        \item We derive and instantiate MA-PoP as a direct estimator of the pooled posterior, implemented as a product of per-agent posteriors. The estimator is exact under conditionally independent private information and serves as an evidence-accumulation approximation when agents are correlated.
        \item We evaluate MA-PoP across six multiple-choice QA benchmarks, covering both homogeneous and heterogeneous MAD configurations, and compare against state-of-the-art debate~\citep{choi2025debate,cui2025free} and voting-based aggregation methods~\citep{ai2025beyond}.
    \end{itemize}
    Our theoretical claims are made within the Blackwell information-structure abstraction and are scoped to multiple-choice QA settings with fixed candidate answers.
    
\section{Background}
\label{sec:background}
\vspace{-2ex}

    This section provides the background on the Blackwell's informativeness framework~\citep{blackwell1953equivalent, de2018blackwell}. To increase the readability, we use calligraphic letters to denote sets, lowercase letters to represent variables, and Greek letters to indicate functions.

    \begin{definition}[Decision-making problem]
        A decision making problem is denoted as a tuple of \((\mathcal{S}, \mathcal{A}, \varphi, \rho)\), where \(\mathcal{S}\) is the set of possible states, \(\mathcal{A}\) is the set of possible actions, \(\varphi: \mathcal{A} \times \mathcal{S} \to \mathbb{R}\) is a utility function, and \(\rho: \mathcal{S} \to [0, 1]\) is a prior distribution over states (i.e., \(\sum_{s \in \mathcal{S}} \rho(s) = 1\)).
    \end{definition}

    \begin{definition}[Information structure]
        An information structure is a tuple \((\mathcal{D}, \sigma)\), where \(\mathcal{D}\) denotes the space of observed signal, and \(\sigma: \mathcal{S} \to [0, 1]\) represents the conditional distribution: \(\sigma(d | s) = \Pr(d | s)\), with \(d \in \mathcal{D}, s \in \mathcal{S}\).
    \end{definition}

        Given the likelihood \(\sigma(d | s)\) of the observed data \(d \in \mathcal{D}\) and the prior over states \(\rho(s)\), the posterior of the state can be obtained via Bayes' rule as 
        \(\Pr(s | d) = \nicefrac{\sigma(d | s) \rho(s)}{\Pr(d)}\), where \(\Pr(d)\) is the prior distribution of observed data \(d\).

    \begin{definition}[Expected value of information structure]
    \label{def:expected_value_information_structure}
        The expected value of an information structure for the decision maker is defined as the expected value of the utility function under the optimal action over states and observations:
        \begin{equation}
        \scalebox{0.9}{$
            \begin{aligned}[b]
                \omega(\sigma) & = \textstyle \sum_{d \in \mathcal{D}} \Pr(d) \underbrace{\left[ \textstyle \max_{a \in \mathcal{A}} \sum_{s \in \mathcal{S}} \varphi(a, s) \Pr(s | d) \right]}_{\mathclap{\text{optimal utility conditioned on observed data}}} = \textstyle\sum_{d \in \mathcal{D}} \max_{a \in \mathcal{A}} \sum_{s \in \mathcal{S}} \varphi(a, s) \, \sigma(d | s) \, \rho(s).
            \end{aligned}
        $}
            \label{eq:expected_value_information_structure}
        \end{equation}
    \end{definition}

    \vspace{-1ex}
    Intuitively, the \emph{expected value} of an information structure can be interpreted as follows:
    \begin{enumerate}[noitemsep,topsep=0pt,leftmargin=5ex]
        \item for each observation \(d\), select the action to maximise the expectation of the utility w.r.t. \(\Pr(s | d)\),
        \item weight that utility by the probability of observing \(d\),
        \item sum (or integrate) over all observations.
    \end{enumerate}

    \begin{definition}[Garbling]
    \label{def:garbling}
        An information structure \((\mathcal{D}^{\prime}, \sigma^{\prime})\) is a \emph{garbling} of \((\mathcal{D}, \sigma)\), denoted as: \((\mathcal{D}^{\prime}, \sigma^{\prime}) \unlhd (\mathcal{D}, \sigma)\), if there exists a function \(\kappa:\mathcal{D} \times \mathcal{D}^{\prime} \to [0, 1]\) such that: \(\sigma^{\prime}(d^{\prime} | s) = \sum_{d \in \mathcal{D}} \kappa(d, d^{\prime}) \, \sigma( d | s)\) and \(\sum_{d' \in \mathcal{D}'} \kappa(d, d') = 1\).
    \end{definition}

    Intuitively, garbling introduces a stochastic transformation that adds noise to the original signal distribution.  Consequently, the expected utility for any decision problem cannot increase under garbling; instead, it can only stay the same or decrease.

    Here, we state the ``partial'' result (excluding the \emph{feasibility}) of the Blackwell's informativeness theorem as follows:
    \begin{theorem}[``Partial'' Blackwell's informativeness~\citep{blackwell1953equivalent}]
    \label{thrm:blackwell}
        If \((\mathcal{D}, \sigma)\) and \((\mathcal{D}^{\prime}, \sigma^{\prime})\) are two information structures,  the following statements are equivalent:
        \begin{enumerate}[topsep=0ex, itemsep=0ex, leftmargin=5ex]
            \item \((\mathcal{D}', \sigma^{\prime})\) is a garbling of \((\mathcal{D}, \sigma)\): \((\mathcal{D}^{\prime}, \sigma^{\prime}) \unlhd (\mathcal{D}, \sigma)\).
            \item Every Bayesian agent prefers \((\mathcal{D}, \sigma)\) to \((\mathcal{D}^{\prime}, \sigma^{\prime})\) for any \if possible \fi decision problem (i.e., \(\omega(\sigma^{\prime}) \le \omega(\sigma)\)).
        \end{enumerate}
        \vspace{-1ex}
    \end{theorem}
    Intuitively, a \say{noisier} information structure results in a lower expected utility. 

\section{Problem statement}
\label{sec:problem_statement}
\vspace{-2ex}

    

    We study multi-agent decision making and classification through the standard Bayesian-game setting~\citep{osborne1994course}, in which each Bayesian agent starts from a non-informative prior over states, \(\rho(s)\), assigning equal probability to all possible outcomes. After receiving private information \(d_m\), each agent updates its belief via Bayesian inference, yielding the posterior \(\Pr(s \mid d_m)\). This abstraction closely resembles the multi-LLM setting, where agents share an uninformative prior before training but acquire distinct private information through their training data or observations. \cref{tab:analogy} maps the decision-theoretic terms used here to their classification counterparts.

    \begin{wraptable}{r}{0.5\textwidth}   
        \vspace{-3ex}
          \centering
           \caption{Mapping table between the terms used in decision-making problem and their interpretation in classification.}
        \label{tab:analogy}
        \resizebox{\linewidth}{!}{        
        \begin{tabular}{l l}
            \toprule
            \bfseries Decision theory & \bfseries Classification \\
            \midrule
            state \(s \in \mathcal{S}\) & true class label \\
            action \(a \in \mathcal{A}\) & predicted label \\
            utility \(\varphi(., .)\) & negative loss function \\
            prior \(\rho(s)\) & class prior distribution \\
            observation \(d \in \mathcal{D}\) & training dataset \\
            info. structure \(\sigma(d | s)\) & class-conditional distribution \\
            \bottomrule
        \end{tabular}}
        \vspace{-2ex}
    \end{wraptable}

    For classification tasks, let \(y \in \mathcal{Y}\) represent label of an instance \(x \in \mathcal{X}\), and \(d\) denote private information (e.g., private training data) observed by an agent. The prediction made by an agent indexed by \(m\), whose private information (or training data) is denoted as \(d_{m}\), for the instance \(x\) can be written as \(\Pr(y | x, d_{m})\). To relate to the decision-making process in \cref{sec:background}, the label of an instance in classification corresponds to the world state in the Blackwell's informativeness setting (i.e., \(\mathcal{S} \equiv \mathcal{Y}\)), the action as the predicted label outcome of the system, and the utility is the negative loss function. Readers are referred to \cref{tab:analogy} for further details.

    The setting considered in this work is the decision-making with multiple agents in classification, where each agent indexed by \(m\) is trained on a private training dataset \(d_{m}\). The objective is to make an informative decision regarding the label \(y\) of a testing instance \(x\). Two prominent approaches are commonly employed in the literature: \emph{(i)} majority voting in which agents independently produce predictions that are subsequently aggregated into a final decision, and \emph{(ii)} debate, where agents exchange their opinions, and update their beliefs accordingly over multiple rounds of interaction.
    
    The objective of this work is to determine whether these two mechanisms, voting and debate, are sufficiently informative under the Blackwell's informativeness framework, or whether an alternative decision-making mechanism may yield a better informative outcome.

\section{Methodology}
\label{sec:method}
\vspace{-2ex}

    This section presents our aggregation result and applies it to classification with multiple LLM agents.

    \subsection{Information aggregation is a garbling}
    \label{sec:information_aggregation}
        We leverage \cref{thrm:blackwell} to formally prove that any aggregation on the private information, denoted as \(g \in \mathcal{G} = \{g: g = \gamma(d_{1}, \dots, d_{M})\}\), is a garbling (or ``noisy'' version) of \((d_{1}, \dots, d_{M})\),
        where \(\gamma: \mathcal{D}^{M} \to \mathcal{Z}\) is an aggregation function and \(\mathcal{Z}\) is an arbitrary set. To simplify the notations, we denote \( d_{1:M}\) as \((d_{1}, \dots, d_{M})\). 
        The following proposition is a direct consequence of the definition of garbling, but it is the key step that lets us instantiate Blackwell's ordering for multi-agent aggregation.
        \begin{proposition}
        \label{lemma:aggregation_garbling}
            If \((\mathcal{G}, \sigma_{\text{g}}(g | s))\) is an information structure derived from the private information observed by each of \(M\) agents with \(\sigma_{\text{g}}(g | s) = \Pr(g | s)\), and \((\mathcal{D}^{M}, \sigma_{0}( d_{1:M} | s))\) is the information structure of original private information with \(\sigma_{0}(d_{1:M} | s) = \Pr(d_{1:M} | s)\), then the following holds:
            \vspace{-1ex}
            \[
                (\mathcal{G}, \sigma_{\text{g}}(g | s)) \unlhd (\mathcal{D}^{M}, \sigma_{0}( d_{1:M} | s)).
            \]
        \end{proposition}

        \vspace{-3ex}
        \begin{proof}
            The proof is based on the definition of probability and the sum rule as follows:
            \begin{equation*}
                \begin{aligned}[b]
                    \sigma_{\text{g}}(g | s) & =  \Pr(g \in \mathcal{G} | s) = \Pr(d_{1:M}' \in \gamma^{-1}(d_{1:M}) | s)  = \textstyle\sum_{d_{1:M}' \in \mathcal{D}^{M}} \kappa(d_{1:M}, d_{1:M}') \sigma_{0}(d_{1:M}' | s),
                \end{aligned}
            \end{equation*}
            where \(\kappa(d_{1:M}, d_{1:M}') = \pmb{1}\left[ \gamma(d_{1:M}') = \gamma(d_{1:M}) \right] \) denotes a valid stochastic mapping function. Hence, according to \cref{def:garbling}, \((\mathcal{G}, \sigma_{\text{g}}(g | s)) \unlhd (\mathcal{D}^{M}, \sigma_{0}( d_{1:M} | s))\).
        \end{proof}

        Intuitively, \cref{lemma:aggregation_garbling} shows that no aggregation mechanism can outperform the pooled information structure of all private information under the Blackwell ordering framework. 
        The novelty of the above theoretical results is not the garbling theorem, but its instantiation in multi-agent decision making.

    \subsection{Decision making in multi-agent classification}
    \label{sec:classification}
        

        In classification with an ensemble of \(M\) agents, aggregation is often performed in the label space, meaning \(\mathcal{Z} \equiv \mathcal{Y}\). For example, in a weighted averaging the output prediction can be calculated as follows:
        \begin{equation}
            g_{\text{weighted}} = \gamma_{\text{weighted}}(d_{1:M}) = \textstyle \sum_{m = 1}^{M} w_{m} \Pr(y | x, d_{m}),
            \label{eq:weighted_average}
        \end{equation}
        where: \(w_{m} \in [0, 1]\) and \(\sum_{m = 1}^{M} w_{m} = 1\). When \(w_{m} = \nicefrac{1}{M}\), this resembles the ensemble learning (also known as single-round-debate voting).

        We employ \cref{lemma:aggregation_garbling} to show that the decision obtained via either aggregated information sources or debate is no more informative under the Blackwell's framework than the decision made from the information source of pooled original private observations.

        \begin{restatable}{proposition}{optimalInformativeness}
        \label{corollary:optimal_informativeness}
            Given \(M\) agents, each observing private information \(d_{m} \in \mathcal{D}, m \in \{1, \dots, M\}\), the prediction probability obtained either through:
            \begin{itemize}[noitemsep,topsep=0pt]
                \item an ensemble of those agents (e.g., weighted average over an ensemble), or
                \item multi-round debate of \(M\) Bayesian agents,
            \end{itemize}
            will result in an information structure that is the garbling of \((\mathcal{D}^{M}, \sigma_{0}( d_{1:M} | y, x))\).
        \end{restatable}

        \vspace{-1ex}
        This result holds independently of protocol details, equilibria, debate structure, or learning dynamics. We, therefore, intentionally omit the details of communication protocols or equilibrium formation within multi-agent systems. For our purposes, it suffices to observe that common collective-decision procedures, such as voting, or debate, are instances of information aggregation mappings applied to agents' private signals. Within this simplified framework, the following is an informal proof sketch of the corollary.
        \begin{proof}[Proof sketch]
            For ensemble learning, as shown in \cref{eq:weighted_average}, the aggregated prediction probability \(g_{\text{weighted}}\) belongs to the set of aggregation of private information \(\mathcal{G}\). Hence, according to \cref{thrm:blackwell} and \cref{lemma:aggregation_garbling}, it is a garbling of \((\mathcal{D}^{M}, \sigma_{0}( d_{1:M} | y, x))\).

            The debate can also be proved similarly. 
            A detailed proof is presented in \cref{appendix:proof_optimal_informativeness}.
        \end{proof}

        As a result, in any classification, every Bayesian decision maker always prefers the information source \((\mathcal{D}^{M}, \sigma_{0}( d_{1:M} | y, x))\) to either the information source obtained via label aggregation or debate. The expected utility value of the information source \((\mathcal{D}^{M}, \sigma_{0}( d_{1:M} | y, x))\) can be written as follows:
        \begin{equation}
            \textstyle \omega(\sigma_{0}) = \sum_{d_{1:M} \in \mathcal{D}^{M}} \max_{\hat{y} \in \mathcal{Y}} \varphi(\hat{y}, y) \sigma_{0}(d_{1:M} | y, x) \rho(y | x),
        \end{equation}
        where \(\hat{y}\) denotes the predicted label. Since \(\varphi(\hat{y}, y)\) is the negative-loss function (see the interpretation in \cref{tab:analogy}), its maximal value is: \(\varphi^{*} = \varphi(\hat{y} = y, y)\). The expected utility can, therefore, be rewritten as follows:
        \begin{equation}
            \textstyle \omega(\sigma_{0}) = \sum_{d_{1:M} \in \mathcal{D}^{M}} \varphi^{*} \times \max_{y \in \mathcal{Y}} \sigma_{0}(d_{1:M} | y, x) \rho(y | x).
            \label{eq:expected_utility_classification}
        \end{equation}

        \cref{eq:expected_utility_classification} means that in a multi-agent classification, for any set of private datasets \(d_{1:M} \in \mathcal{D}^{M}\), every Bayesian decision maker prefers having the predicted label as follows:
        \begin{equation}
            \begin{aligned}[b]
                \hat{y}^{*} & = \textstyle \operatorname*{argmax}_{y \in \mathcal{Y}} \sigma_{0}(d_{1:M} | y, x) \rho(y | x) \textstyle = \operatorname*{argmax}_{y \in \mathcal{Y}} \Pr(y, d_{1:M} | x)  \textstyle = \operatorname*{argmax}_{y \in \mathcal{Y}} \Pr(y | x, d_{1:M}).
            \end{aligned}
            \label{eq:predicted_label}
        \end{equation}

        \begin{remark}
            The result in \cref{eq:predicted_label} also helps explain the diminishing returns and eventual performance plateaus observed in several multi-LLM debate settings. In homogeneous ensembles, where agents share pre-training corpora, alignment procedures, and retrieval pipelines, their private knowledge is positively correlated, reducing the diversity of joint private knowledge \(\cap_{ = 1}^{M} d_{m}\), and hence, performance is predictably lower than in heterogeneous configurations as reported in \citep{zhang2025stop}. Conversely, even in heterogeneous pools, increasing the number of agents beyond a modest threshold yields limited gains once their training data cease to contribute novel information, leading to the plateau observed in \cref{tab:four_agents_PoP}. A related saturation effect arises with larger LLMs: because their training corpora overlap substantially, adding more large models does not necessarily increase knowledge diversity. Our findings in \cref{eq:predicted_label} suggest that, for pooled posterior methods, the most effective configuration involves a small number of agents whose private knowledge is as close to conditionally independent as feasible. Ensuring diversity is therefore crucial to realise the theoretical advantages of pooled aggregation in multi-agent decision-making systems.
        \end{remark}

        The result in \cref{eq:predicted_label} is also known as maximum Bayesian pooled posterior. 
        Within this Blackwell information-structure abstraction, voting and debate are less informative than access to the full pooled private information. 
        Next, we present a practical method to approximate the pooled posterior for decision making in multi-agent classification.

    \paragraph{Estimate the pooled posterior}
    If all of the private information observed by each agent, \(d_{m}, m \in \{1, \dots, M\}\) is available, it is straight-forward to obtain \(\Pr(y | x, d_{1:M})\). This is, however, not true in practice because the private information (e.g., private training data) \(d_{m}\) of each agent is often unavailable to other agents \(m^{\prime} \in \{1,\dots,M\} \setminus \{m\}\). Instead, it is only possible to sample the prediction made by each agent's posterior \(\Pr(y | x, d_{m})\). If this sampling process is reasonably cheap, one can repeatedly sample multiple predictions of one agent to approximate its posterior, denoted as \(\Tilde{\Pr}(y | x, d_{m})\). In the following, we assume the availability of \(\Tilde{\Pr}(y | x, d_{m}), m \in \{1, \dots, M\}\), and use these to approximate the pooled posterior \(\Pr(y | x, d_{1:M})\).

    Given the availability of each agent's approximated posterior \(\Tilde{\Pr}(y | x, d_{m})\), the pooled posterior of interest can then be estimated by applying Bayes' rule as follows:
        \begin{equation}
        \begin{aligned}[b]
            \Pr(y | x, d_{1:M}) 
            &\propto \Pr(y|x) \textstyle\prod_{m=1}^{M} \Pr(d_m|y,x) 
             \quad \text{assume } (d_m \perp d_{m'} | y,x,\ m \neq m') \\
            &\propto \Pr(y|x) \textstyle\prod_{m=1}^{M} \frac{\Pr(y|x,d_m)\Pr(d_m)}{\Pr(y|x)} \\
            &\propto [\underbrace{\Pr(y|x)}_{\text{uniform}}]^{1-M} \textstyle\prod_{m=1}^{M} \Pr(y|x,d_m) \\
            &\propto \textstyle\prod_{m=1}^{M} \Pr(y|x,d_m) \approx \prod_{m=1}^{M} \tilde{\Pr}(y|x,d_m).
        \end{aligned}
        \label{eq:PoP}
    \end{equation}
        The result in \cref{eq:PoP} shows that the pooled posterior is approximately proportional to the \emph{product of each agent's posterior} (up to a normalisation constant). We adopt two standard assumptions in the information and LLM aggregation literature~\citep{ai2025beyond, prelec2017solution, schoenebeck2021wisdom, pan2024robust, pettigrew2019aggregating}: \emph{(i)} the conditional independence of agents' private knowledge, and \emph{(ii)} uniform prior \(\Pr(y | x)\). The conditional independence assumption yields a closed-form approximation of the Bayesian posterior pooling and is particularly appropriate in heterogeneous settings where agents acquire information from distinct sources or specialise in different domains (e.g., software engineering versus biology or law). 
        \cref{eq:PoP} should therefore not be read as an unconditional optimality guarantee. The pooled posterior in \cref{eq:predicted_label} is Blackwell-optimal when the joint private information \(d_{1:M}\) is available, whereas the product estimator in \cref{eq:PoP} is exact only under conditional independence of agents' private information. We treat this assumption not as a factual claim about modern LLMs, but as an idealised reference point. As agent dependence increases (e.g., overlapping of pre-training and fine-tuning data), the pooled posterior estimated in \cref{eq:PoP} is overconfident. Empirically, this overconfidence does not lead to a catastrophic degradation, but a smooth performance changes as observed in
        \cref{tab:qwen_PoP,tab:llama_PoP,tab:qwen32_PoP}.
        
        The assumption of uniform prior \(\Pr(y | x)\) in \cref{eq:PoP} should be read as the absence of external information about class label distribution. If a prior \(\Pr(y\mid x)\) is available, it can be incorporated directly through the more general form \(\Pr(y\mid x,d_{1:M}) \propto \Pr(y\mid x)^{1-M}\prod_m \Pr(y\mid x,d_m)\).

    \paragraph{Approximate \(\Tilde{\Pr}(y | x, d_{m})\) of LLM agents in multiple-choice questions}
        We employ the similarity of embeddings between the agent's response \(y^{(m)} \sim \Pr(y | x, d_{m})\) and each of \(J\) provided answers (or options) \(y_{o_{j}}, j \in \{1, \dots, J\}\) in a multiple-choice question (e.g., \(y_{o_{j}}\) is a natural language expression as: A-benign, B-malignant or C-no finding as in \cref{appendix:prompt_templates}) to approximate \(\Tilde{\Pr}(y | x, d_{m})\).
        Formally, for each pair of agent's response \(y^{(m)}_{n}\) and a provided list of options, \(y_{o_{j}}\), we calculate the similarity of their embeddings, denoted as \(\tau(y^{(m)}_{n}, y_{o_{j}})\). For example, \(\tau\) can be the dot product or cosine similarity between the two embedding vectors of \(y^{(m)}_{n}\) and \(y_{o_{j}}\). The predictive probability distribution over all provided options can be approximated as follows:
        \begin{equation}
            \textstyle \Tilde{\Pr}(y = y_{{o}_{j}} | x, d_{m}) = \frac{1}{N} \sum_{n = 1}^{N} \nicefrac{\phi \left( \tau(y^{(m)}_{n}, y_{o_{j}}) \right)}{\sum_{j' = 1}^{J} \phi \left( \tau(y^{(m)}_{n}, y_{o_{j'}}) \right)},
            \label{eq:approximate_posterior}
        \end{equation}
        where \(N\) is the number of \say{Monte Carlo} responses for the same question, and the additional model \(\phi(\cdot)\) is introduced to map the similarity score to a predictive logit (e.g., plays a role as Platt scaling~\citep{platt1999probabilistic}). Note that we estimate each agent's posterior from full generated responses by measuring their semantic compatibility with each candidate option using an NLI cross-encoder, followed by calibration. This avoids relying on raw LLM logits, whose token-level probabilities over labels or answer strings can be noisy, length-biased, and misaligned with full-answer semantics; see Appendix~\ref{appendix:logits_embedding_comparison} for details.

        Given the approximation of each agent's posterior in \cref{eq:approximate_posterior}, we can approximate the pooled posterior \(\Pr(y | x, d_{1:M})\) by multiplying these probabilities and normalising to obtain the predicted label (or the selected answer) as shown in \cref{eq:predicted_label}.

        Note that in the conventional multi-class classification, the order of appearance of class labels is fixed. However, in the setting of multiple choice questions, the \say{class labels} or provided options can be freely permutated. Hence, the model \(\phi(\cdot)\) used for calibrating the prediction probabilities in \cref{eq:approximate_posterior} must be permutation-equivariant. 
        In our implementation, we formulate \(\phi\) as a fully-connected neural network following the \emph{Deep Sets} architecture~\citep{zaheer2017deep} to make the prediction permutation-equivariant.

\section{Experiments}
\label{sec:experiment}

\paragraph{Baselines}
We follow the setting in \citep{choi2025debate} to evaluate MA-PoP compared to majority voting and MAD. The MAD approaches include: \emph{Centralised MAD}~\citep{guo2024large}, where a central agent aggregates peer responses and generates the updated response at each round, \emph{Decentralised MAD}~\citep{du2023improving}, where each agent observes all other agents' responses from the previous round, \emph{Sparse MAD}~\citep{li2024improving}--a variant of Decentralised MAD with a sparse communication topology to enhance efficiency, and \emph{Free MAD}~\citep{cui2025free}, where decisions are made by evaluating the entire trajectory rather than the final consensus. The voting-based approach includes: \emph{majority voting}, \emph{Self-Consistency} (SC)~\citep{wang2023self}, which samples \(N\) responses from each agent and selects the most frequent answer across all agents, \emph{Log-linear opinion pooling} (LLP)~\citep{pettigrew2019aggregating}, which relaxes the conditional independence assumption by assigning learnable per-agent weights, and \emph{Inverse Surprising Popularity} (ISP)~\citep{ai2025beyond}, which leverages first-order information (e.g., agent accuracies) and second-order information (e.g., answer correlations) to aggregating LLMs' responses. Each agent is sampled with the same number of responses, \(N\), to fairly compare.

We consider a pool of \(M = 5\) agents for multi-agent methods in our main comparison and will ablate on the effect of \(M\). For single-agent baselines, we average across five independent runs. \texttt{Qwen2.5-7B-Instruct} \citep{yang2025qwen3}, \texttt{Llama-3.1-8B-Instruct} \citep{grattafiori2024llama}, \texttt{Mistral-7B-Instruct-v0.3} \citep{jiang2023mistral}, \texttt{Falcon-H1-7B-Instruct} \citep{falconh1}, and \texttt{Gemma-2-9B-Instruct} \citep{gemma_2024} are used for heterogeneous pool. In the homogeneous setting (i.e., all agents are instances of the same LLM), MA-PoP is simplified to maximising the posterior of a single agent, while in the heterogeneous setting, MA-PoP is carried out following \cref{algorithm:PoP_MAD} in \cref{appendix:algorithm}. We also keep the same number of Monte Carlo samples \(N = 5\) throughout all the experiments. For all MAD baselines, we report results with \(T \in \{2, 3, 5\}\) debate rounds.
We treat the heterogeneous-agent experiments as the primary test of multi-agent information pooling, and the homogeneous-agent experiments as a test of the per-agent posterior estimation pipeline (\cref{eq:approximate_posterior}).

\paragraph{Benchmarks} consist of six natural language question answering tasks, each associated with the following datasets: \emph{(i)} Factual Question Answering (MMLU Professional Medicine and Formal Logics~\citep{hendrycks2020measuring, estornell2024multi}), \emph{(ii)} Natural Language Inference (HellaSwag~\citep{zellers2019hellaswag}), \emph{(iii)} Commonsense Reasoning (CSQA~\citep{talmor2019commonsenseqa}), \emph{(iv)} Alignment Labeling (HH-RLHF~\citep{bai2022training}), where we adopt the ``AI labeller alignment'' practice~\citep{lee2024rlaif}, similarly to~\citep{choi2025debate}, and
\emph{(v)} Medical Entrance Exam (MedMCQA)~\citep{pal2022medmcqa}.
For a fair comparison, all baselines are evaluated on the same data subsets. Readers are referred to \cref{appendix:experimental_details} for further details of the datasets used.

\begin{table*}[t]
\centering
\caption{Prediction accuracy of \(M = 5\) agents in the heterogeneous setting across six QA benchmarks. The results consist of the mean and standard deviations obtained from three different random seeds. The best result per dataset is highlighted in \textbf{bold}.}
\resizebox{\linewidth}{!}{
\begin{tabular}{l cccccc}
\toprule
 & \textbf{\shortstack{MMLU (Pro.Med.)}} & \textbf{\shortstack{MMLU (Form.Log.)}} & \textbf{HellaSwag} & \textbf{CSQA} & \textbf{HH-RLHF} &\textbf{MedMCQA}\\
\midrule
\multicolumn{7}{c}{\textbf{Single-Agent}} \\
\midrule
Qwen-7B  & 0.7868 \(\pm\) .01 & 0.4905 \(\pm\) .03 & 0.7880 \(\pm\) .01 & 0.8153 \(\pm\) .01 & 0.4773 \(\pm\) .01 & 0.5467 \(\pm\) .01\\
Falcon-7B  & 0.7904 \(\pm\) .01  & 0.5873 \(\pm\) .01   & 0.7133 \(\pm\) .01  & 0.8300 \(\pm\) .01  & 0.5033 \(\pm\) .01  & 0.5767 \(\pm\) .01  \\         
Mistral-7B  & 0.6544 \(\pm\) .01  & 0.3730 \(\pm\) .01   & 0.6433 \(\pm\) .01  & 0.6800 \(\pm\) .01  & 0.5567 \(\pm\) .01  & 0.3133 \(\pm\) .01  \\
Llama-8B  & \multicolumn{1}{c}{0.7441 \(\pm\) .01}  & 0.3794 \(\pm\) .02 & 0.6267 \(\pm\) .03 & 0.6767 \(\pm\) .01 & 0.4440 \(\pm\) .02   & 0.5000 \(\pm\) .01        \\ 
Gemma-9B  & \multicolumn{1}{c}{0.8015 \(\pm\) .01}  & 0.5159 \(\pm\) .02 & 0.8267 \(\pm\) .01 & 0.8033 \(\pm\) .01 & 0.5567 \(\pm\) .02     & 0.5567 \(\pm\) .01\\
\midrule
\multicolumn{7}{c}{\textbf{Multi-Agent}} \\
\midrule
Decentr. MAD (T=2) & 0.8603  \(\pm\) .01 & 0.7063  \(\pm\) .01   & 0.8267  \(\pm\) .01  & 0.8500  \(\pm\) .01  & 0.5867  \(\pm\) .01  & 0.6200  \(\pm\) .01        \\
Decentr. MAD (T=3) & 0.8603  \(\pm\) .00  & 0.6904  \(\pm\) .01   & 0.8333  \(\pm\) .01  & 0.8533  \(\pm\) .01  & 0.5867  \(\pm\) .01  & 0.6067  \(\pm\) .01        \\
Decentr. MAD (T=5) & 0.8676  \(\pm\) .00  & 0.6984  \(\pm\) .02   & 0.8400  \(\pm\) .01  & 0.8433  \(\pm\) .00  & 0.5700  \(\pm\) .00  & 0.6000  \(\pm\) .01        \\
Centr. MAD (T=2)   & 0.8372  \(\pm\) .02  & 0.5635  \(\pm\) .03   & 0.7667  \(\pm\) .02  & 0.8233  \(\pm\) .01  & 0.5533  \(\pm\) .01  & 0.5900  \(\pm\) .02            \\
Centr. MAD (T=3)   & 0.8333  \(\pm\) .01  & 0.6032  \(\pm\) .02   & 0.7567  \(\pm\) .01  & 0.8067  \(\pm\) .00  & 0.5567  \(\pm\) .01  & 0.5833  \(\pm\) .02            \\
Centr. MAD (T=5)   & 0.8419  \(\pm\) .01  & 0.6032  \(\pm\) .03   & 0.7633  \(\pm\) .02  & 0.8233  \(\pm\) .01  & 0.5367  \(\pm\) .00  & 0.5967  \(\pm\) .01             \\
Sparse MAD (T=2)        & 0.8603  \(\pm\) .00  & 0.6667  \(\pm\) .02   & 0.8200  \(\pm\) .01  & 0.8633  \(\pm\) .00  & 0.5733  \(\pm\) .00  & 0.6200  \(\pm\) .01        \\
Sparse MAD (T=3)        & 0.8529  \(\pm\) .00  & 0.7143  \(\pm\) .00   & 0.8133  \(\pm\) .00  & 0.8533  \(\pm\) .01  & 0.5533  \(\pm\) .01  & 0.6033  \(\pm\) .01        \\
Sparse MAD (T=5)        & 0.8529  \(\pm\) .01  & 0.7084  \(\pm\) .00   & 0.8133  \(\pm\) .01  & 0.8500  \(\pm\) .00  & 0.5600  \(\pm\) .00  & 0.6167  \(\pm\) .02        \\
Free MAD (T=2)        & 0.8713  \(\pm\) .00   & 0.6111  \(\pm\) .03  & 0.7900  \(\pm\) .02  & 0.8533  \(\pm\) .00  & 0.5500  \(\pm\) .03  & 0.6133  \(\pm\) .02            \\
Free MAD (T=3)        & 0.8713  \(\pm\) .00   & 0.6111  \(\pm\) .04  & 0.7967  \(\pm\) .02  & 0.8567  \(\pm\) .00  & 0.5567  \(\pm\) .02  & 0.6200  \(\pm\) .00            \\
Free MAD (T=5)        & 0.8750  \(\pm\) .01   & 0.6270  \(\pm\) .03  & 0.7933  \(\pm\) .02  & 0.8467  \(\pm\) .01  & 0.5667  \(\pm\) .00  & 0.6267  \(\pm\) .00            \\
\midrule
Majority Voting & 0.8493  \(\pm\) .01  & 0.6190  \(\pm\) .01   & 0.8033  \(\pm\) .01  & 0.8533  \(\pm\) .00  & 0.5633  \(\pm\) .01  & 0.6167  \(\pm\) .01 \\
Self-Consistency & 0.8603 \(\pm\) .01   & 0.7084  \(\pm\) .01  & 0.8233  \(\pm\) .00  & 0.8600  \(\pm\) .00  & 0.5733  \(\pm\) .01  & 0.6167  \(\pm\) .01 \\
LLP & 0.8567  \(\pm\) .02   & 0.7201  \(\pm\) .01  & 0.7833  \(\pm\) .00  & 0.8667  \(\pm\) .00  & 0.5367  \(\pm\) .01  & 0.6321  \(\pm\) .00 \\
ISP & 0.8640  \(\pm\) .00   & 0.6429  \(\pm\) .01  & 0.8267  \(\pm\) .00  & 0.8600  \(\pm\) .00  & 0.5667  \(\pm\) .01  & 0.6367  \(\pm\) .01 \\
\midrule
\textbf{MA-PoP}  & \textbf{0.8787 $\pm$ .01}  & \textbf{0.7367 \(\pm\) .01}   & \textbf{0.8433 \(\pm\) .01}  & \textbf{0.8800 \(\pm\) .00}  & \textbf{0.5900 \(\pm\) .01}  & \textbf{0.6467 \(\pm\) .01}       \\
\bottomrule
\end{tabular}
}
\label{tab:five_agents_PoP}
\vspace{-2ex}
\end{table*}

\paragraph{Results on standard MAD benchmarks} are shown in
\cref{tab:five_agents_PoP} with five different LLMs across five benchmarks, while \cref{fig:comparison} in \cref{appendix:performance_gap} shows the gap among the best single model/ SOTA MAD, majority voting, Self-Consistency, LLP, ISP and MA-PoP.
Empirical results demonstrate that multi-agent methods consistently improve performance over single-agent baselines across most datasets. However, the degree of improvement varies significantly by task, highlighting the sensitivity of these methods to the underlying task complexity.

MA-PoP achieves the highest performance across all evaluated tasks. Although majority voting has strong performance in homogeneous agent settings~\citep{choi2025debate}, it is not on par with MAD methods in heterogeneous settings and sometimes even underperforms individual agents (e.g., see HellaSwag in \cref{tab:five_agents_PoP}). Furthermore, other debate methods show inconsistent performance across different numbers of rounds (\(T=2, 3, 5\)). On some datasets like HellaSwag and HH-RLHF, additional debate rounds provide marginal gains or even degrades performance compared to naive voting. These findings empirically validates recent critiques~\citep{choi2025debate, kaesberg2025voting} regarding the diminishing returns of iterative debate compared to voting (or single-turn ensembles). Notably, on HellaSwag, where traditional debate methods often underperform the best single agent, MA-PoP successfully leverages complementary knowledge of multiple agents, exceeding single-agent performance. This demonstrates that the limitations of conversational debate stem not from multi-agent collaboration itself, but from suboptimal aggregation strategies.

Across different debate rounds, traditional MAD methods show diminishing or even negative returns as the debate progresses. 
This is consistent with the broader finding of~\citep{laban2026llms} that LLMs exhibit significantly lower performance in multi-turn conversations than single-turn. Multi-turn debate inherits the same conversational-degradation pathologies.
In contrast, MA-PoP achieves its superior results in a single turn by directly aggregating the initial probabilistic beliefs, providing a more computationally efficient and stable alternative to iterative communication.

The LLP baseline provides a dependence-aware reference point: under conditional independence, MA-PoP corresponds to a log-linear pool with all per-agent weights fixed to one, whereas LLP learns these weights from a small validation set (about 300 pooled examples, see \cref{appendix:experimental_details}). MA-PoP's gains over LLP should not be read as evidence that conditional independence holds; rather, in the low-supervision regime considered here, fixed unit weights provide a useful inductive bias relative to learning global per-agent weights from limited data.

\paragraph{Heterogeneous and Homogeneous Pool of LLMs}
To assess robustness to dependence, we evaluate MA-PoP across a spectrum of agent-diversity regimes: heterogeneous agents, reduced heterogeneous pools, homogeneous agents, and larger models with likely overlapping training data.
\cref{tab:four_agents_PoP} demonstrates that MA-PoP maintains strong performance even with reduced agent diversity, demonstrating its robustness to ensemble size (see \cref{appendix:num_heterogeneous_agents}). However, we also observe that increasing the number of agents beyond a certain threshold does not yield proportional performance gains across all datasets. This limitation stems from the conditional independence assumption on agents' private training data imposed in \cref{eq:PoP}. While introducing additional agents can initially enrich the collective knowledge, thereby improving performance, this effect diminishes once the pool of agents begins to contribute overlapping information rather than genuinely novel evidence. Beyond this point, adding more agents no longer enhances informativeness and instead leads to a performance plateau.

\cref{tab:two_agents_PoP} in \cref{appendix:num_heterogeneous_agents} presents results at the extreme setting consisting of only 2 agents (\texttt{Qwen-7B} and \texttt{Llama-8B}). This is noteworthy because majority voting fundamentally fails when two agents disagree. In that case, majority voting randomly selects one out of the two answers, resulting in an unreliable decision. In contrast, MA-PoP can, however, leverage the strength of each agent's posterior distribution to make informed decisions without the need of consensus.
Empirically, MA-PoP performs competitively and achieves the best results on most benchmarks in this setting, suggesting that posterior-level aggregation can remain useful even when majority voting becomes unreliable.

Following \citep{choi2025debate}, \cref{tab:llama_PoP,tab:qwen_PoP} in \cref{appendix:homogeneous_agents} evaluate MA-PoP in the homogeneous setting (e.g., all 5 agents are instances of the same LLM). In this setting, the pooled posterior in MA-PoP is simplified into a single posterior because all these agents are the same. 
Nevertheless, MA-PoP still outperforms both single-agent baselines and traditional MAD methods.
Notably, Centralised MAD exhibits severe degradation with additional rounds of debate, highlighting the risks of iterative debate when agents share the same model and may amplify shared biases~\citep{zhang2025stop}.
Across these heterogeneous and homogeneous regimes, performance gains decrease as redundancy increases, but do not collapse.

\paragraph{Larger and more capable models}
We extend the evaluation by employing LLMs with larger capacity. Specifically, for the homogeneous setting, \texttt{Qwen2.5-32B-Instruct}~\citep{yang2024qwen25} is used as a base model. As shown in~\cref{tab:qwen32_PoP} (see \cref{appendix:larger_models}), the performance of majority voting remains comparable to that of multi-agent methods. For the heterogeneous setting, \texttt{Qwen2.5-32B-Instruct}~\citep{yang2025qwen3}, \texttt{Falcon-H1-34B-Instruct}~\citep{falconh1}, and \texttt{Gemma-2-27B-Instruct}~\citep{gemma_2024} are used for 3-agent debate. The results in \cref{tab:diff3_PoP} (see \cref{appendix:larger_models}) for Centralised MAD agree with our previous evaluation in \cref{tab:five_agents_PoP,tab:four_agents_PoP}, in which Centralised MAD cannot improve performance and is even worse than single models and majority voting. Other MAD methods perform better than single agent but are inferior compared to our proposed method MA-PoP.

\paragraph{Calibration of MA-PoP}
We evaluate calibration quality by measuring Expected Calibration Error (ECE) and Maximum Calibration Error (MCE) in \cref{tab:calibration}, with corresponding reliability diagrams shown in \cref{fig:reliability}.  We test four models 
(\texttt{Qwen-7B}, \texttt{Falcon-7B}, \texttt{Gemma-9B}, and \texttt{Falcon-34B}) on the MedMCQA dataset (see \cref{appendix:calibration}), comparing performance with and without our calibration module. Results demonstrate that our method achieves substantial calibration improvements across different model scales, producing well-calibrated probability estimates that enable effective multi-agent consensus mechanisms. 

\paragraph{Ablation of MA-PoP}

We isolate each component of MA-PoP through three ablations against the full method (\cref{tab:ablation_nli} in \cref{appendix:ablation}). (1) \emph{w/o PoP} replaces cross-agent pooling with a single best agent's calibrated posterior, averaged over \(N=5\) Monte Carlo samples (\cref{eq:approximate_posterior}). 
The results show that calibration alone yields part of the gain over the single-agent baseline, while cross-agent pooling provides additional improvements across most benchmarks.
(2) \emph{w/o calibration} replaces the Deep Sets with softmax-normalised NLI logits. Multiplying these uncalibrated scores yields only marginal gains, which confirms that well-formed probabilities are a prerequisite for multiplicative aggregation. (3) \emph{Temperature scaling} replaces the Deep Sets with a single learned scalar temperature~\citep{guo2017calibration,platt1999probabilistic}. This improves over no calibration on some benchmarks (e.g., MMLU Pro.\ Med., CSQA) but underperforms on others (HellaSwag, MMLU Formal Logic) and consistently falls short of the MA-PoP, suggesting that a scalar temperature is insufficient to capture the permutation-equivariant structure of multiple-choice options. Overall, these results indicate that pooling and calibration are complementary, with pooling providing the larger contribution and calibration being necessary to realise it.

\paragraph{Number of Monte Carlo samples} We analyse the impact of the number of Monte Carlo samples, denoted by \(N\), on the quality of posterior estimation in MA-PoP. The results in \cref{fig:numofsamples} (see \cref{appendix:num_samples}) for the homogeneous setting show substantial improvement when increasing from \(N = 1\) to \(N = 5\) samples, with performance stabilising beyond \(N \ge 5\). This demonstrates that posterior estimates converge with additional samples, though with diminishing marginal returns, suggesting \(N = 5\) provides a good balance between estimation quality and computational cost.

\paragraph{Efficiency} We evaluate computational cost by measuring the input, output, and total number of tokens consumed per sample across all methods on the MedMCQA dataset in the 5-agent heterogeneous setting. \cref{tab:token_comparison} in \cref{appendix:computational_efficiency_analysis} reports token usage accumulated across all debate rounds. MA-PoP achieves token efficiency comparable to single-round voting baselines and substantially lower than MAD methods, particularly at higher debate rounds. The additional NLI cross-encoder scoring step required by MA-PoP runs in under one second per sample on a single GPU and is negligible relative to LLM generation cost, which is about 25 seconds.

\section{Related work}
\label{sec:related_work}

\paragraph{Multi-agent debate (MAD)} enables agents to iteratively exchange arguments and critiques over multiple rounds, with the aim of improving reasoning through deliberation and mutual correction~\citep{irving2018ai, du2023improving, liu2024groupdebate}. While debate has shown potential for exposing diverse perspectives and correcting individual errors~\citep{michael2023debate, khan2024debating}, recent empirical studies highlight significant limitations: multi-round debate often fails to outperform strong single-agent baselines and can be less reliable than simple majority voting~\citep{zhang2025stop, choi2025debate}. Free-MAD~\citep{cui2025free} proposes a consensus-free approach where decisions are made by evaluating the entire debate trajectory rather than the final consensus. Within the Blackwell information-structure abstraction, existing voting-based and debate-oriented approaches are less informative than Bayesian pooled posterior, as formalised in \cref{corollary:optimal_informativeness}. In contrast, we ground multi-agent decision making in Blackwell's informativeness framework to generate decisions with informative-theoretic guarantees.

\paragraph{Opinion Pooling and Voting} 
Opinion pooling concerns the aggregation of probabilistic beliefs from multiple experts into a single opinion~\citep{stone1961opinion}. Early work includes the DeGroot~\citep{degroot1974reaching}, in which agents iteratively update their beliefs via weighted linear combinations, corresponding to linear opinion pooling (e.g., \cref{eq:weighted_average}). Beyond linear pooling, log-linear pooling  (LLP) has been proposed~\citep{genest1986combining} to account for dependencies among agents' beliefs. LLP assigns each agent \(m\) a non-negative weight \(w_{m} \ge 0\), yielding a pooled belief of the form \(\prod_{m = 1}^{M} [\Pr(y | x, d_{m})]^{w_{m}}\). Log-linear pooling has also been studied in the context of incoherent opinion pooling through renormalisation procedures~\citep{pettigrew2019aggregating}. 

Voting in an ensemble of LLMs instantiates linear opinion pooling, with majority voting as the most common aggregation rule. To reduce uncertainty, \emph{self-consistency} selects the most frequent answer across multiple sampled prompts or reasoning paths from a single model~\citep{wang2023self, wang2024boosting}. Beyond simple voting, more advanced schemes include mixture-of-experts~\citep{shazeer2017outrageously}, pairwise ranking-based aggregation~\citep{jiang2023llm}, confidence-weighted voting~\citep{xiong2024can}, and hybrids that combine output agreement with internal representation consistency~\citep{jiang2025representation}. Recent work such as \emph{Inverse Surprising Popularity}~\citep{ai2025beyond} further improves aggregation by leveraging agent accuracies (i.e., first-order) and inter-answer (i.e, second-order) correlations.

\cref{sec:extended_related_work} provides further related work, including multi-agent decision-making, LLM uncertainty estimation and \emph{product of experts}~\citep{hinton2002training}.

\section{Discussion and conclusion}
\label{sec:discussion}

In this paper, we study the information generated by voting and debate in multi-agent decision-making. We show that, under Blackwell's informativeness framework, decisions based on these sources are provably no more informative than those based on pooled information. 
This insight motivates Bayesian pooled-posterior maximisation as an information-theoretic reference rule. MA-PoP provides a practical approximation to this rule by multiplying estimated individual posteriors, a form that is exact under conditional independence and appropriate when agents contribute diverse, weakly correlated evidence.

Building on this, we propose MA-PoP, a practical method for approximating agent posteriors in LLM-based settings. MA-PoP consistently outperforms prior state-of-the-art methods on standard debate benchmarks.
A key limitation of this work is that MA-PoP’s product estimator relies on conditional independence of agents’ private information; in realistic LLM deployments, overlapping pre-training, fine-tuning, and toolchains can violate this assumption and lead to overconfident pooled posteriors, as discussed in Appendix~\ref{appendix:limitations}.
In fact, the empirical gains of MA-PoP decrease as agent redundancy increases, but the degradation is gradual, suggesting that the method remains useful as an approximation when independence is only partially satisfied.

\bibliography{example_paper}
\bibliographystyle{unsrtnat}
\newpage
\appendix

\clearpage
\section{Proof of \cref{corollary:optimal_informativeness}}
\label{appendix:proof_optimal_informativeness}
    \optimalInformativeness*
    \begin{proof}
        We prove that the consensus obtained via multi-agent debate is an aggregation or pooling of the private information \(\mathcal{G}\). The proof is carried out by induction.

        In the first round (\(T = 1\)), each agent \(m\) draws an opinion \(y_{m}^{(0)}\) from its probabilistic belief \(\Pr(y | x, d_{m})\) and share that to other agents. The opinion, \(y_{m}^{(0)}\) can also be viewed as a function of the private knowledge \(d_{m}\) as follows:
        \begin{equation}
            y_{m}^{(0)} \sim \Pr(y | x, d_{m}) \implies y_{m}^{(0)} = y_{m}^{(0)} (d_{m}).
        \end{equation}

        After the first round, each agent observes all samples of belief shared by other agents, and update its belief using Bayes's rule as follows:
        \begin{equation}
            \operatorname{Pr}^{(1)} \left( y \left\vert x, d_{m}, \prod_{m' = 1, m' \neq m}^{M} y_{m'}^{(0)}(d_{m'}) \right. \right) \propto \Pr(y | x, d_{m}) \prod_{m' = 1, m' \neq m}^{M} \Pr( \left. y_{m'}^{(0)}(d_{m'}) \right\vert x, y, d_{m}).
        \end{equation}
        Because each opinion \(y_{m}^{(0)}\) is a function of each private information \(d_{m}\), we can represent the posterior under the ``lumped'' function form as follows:
        \begin{equation}
            \operatorname{Pr}^{(1)} \left( y \left\vert x, \gamma_{m}^{(0)}(d_{1:M}) \right. \right),
        \end{equation}
        where \(\gamma_{m}^{(0)}\) is a function.

        In the second round \(T = 2\), we follow the similar procedure and show that the posterior of an agents is as follows:
        \begin{equation}
            \operatorname{Pr}^{(2)} \left( y \left\vert x, d_{m}, \prod_{m' = 1, m' \neq m}^{M} y_{m'}^{(1)}(d_{m'}) \right. \right) = \operatorname{Pr}^{(2)} \left( y \left\vert x, \gamma_{m}^{(1)}(d_{1:M})  \right. \right),
        \end{equation}
        where \(\gamma_{m}^{(1)}\) is a function at round \(T = 2\).

        At the consensus, the posterior of all agents is denoted as: \(\Pr(y | x, \gamma_{\text{con.}}(d_{1:M}))\). The information structure corresponds to this posterior can be written as follows:
        \begin{equation}
            \Pr( \gamma_{\text{con.}}(d_{1:M}) | x, y) \propto \Pr(y | x, \gamma_{\text{con.}}(d_{1:M})) \Pr(\gamma_{\text{con.}}(d_{1:M}) | x).
        \end{equation}

        The information structure \((\mathcal{G}_{\text{con.}}, \Pr( g | x, y))\) with \(\mathcal{G}_{\text{con.}} = \{ g: g = \gamma_{\text{con.}}(d_{1:M}) \}\) is obviously an aggregation of the joint of original private information \(d_{1:M}\). Hence, according to \cref{lemma:aggregation_garbling}, \((\mathcal{G}_{\text{con.}}, \Pr( g | x, y)) \unlhd (\mathcal{D}^{M}, \Pr(d_{1:M} | y, x)\). Applying \cref{thrm:blackwell} on this result completes the proof.
    \end{proof}

\clearpage
\section{Logits vs. Semantic Embedding for posterior probability estimation}
\label{appendix:logits_embedding_comparison}
While using LLM logits directly is an intuitive alternative, there are two limitations that make it not robust enough for reliable posterior estimation in multi-agent setting:

\textbf{Multi-Token Label Ambiguity.} LLM logits correspond to the next-token probability conditioned on the prompt, not the probability of an entire candidate answer. When candidate labels span multiple tokens, the first-token logit is only a partial signal. For example, in HellaSwag, given the context "A man is sitting on a roof. he", the candidate answers are ["is using wrap to wrap a pair of skis.", "is ripping level tiles off.", "is holding a rubik's cube.", "starts pulling up roofing on a roof."]. Computing the full joint probability would require autoregressive decoding over every candidate, which is computationally expensive. It will also require to normalise probabilities by answer length, but this does not provide true probabilities (it effectively measures average per-token probability) and can bias results toward longer answers that contain more common (i.e., higher-probability) words.

\textbf{Signal vs. Noise.} Even for single-token labels (e.g., A/B/C/D in multiple-choice), the LLM's probability mass is distributed across the entire vocabulary, not just the valid candidates. Consider a typical scenario where the first-token distribution assigns overwhelming mass to a task-irrelevant token: the probability of ``The'', ``answer'', ``is'', ``A'', ``B'', ``C'', ``D''] is [0.9, 0.03, 0.03, 0.01, 0.01, 0.01, 0.01]. After masking to valid candidates and renormalising, we obtain a uniform distribution [0.25, 0.25, 0.25, 0.25] that has lost all discriminative signal. LLMs frequently allocate dominant probability mass to high-frequency tokens that are irrelevant to the task, leaving the valid-label logits in a low-signal regime where noise dominates.

Instead, our method captures the semantics of the complete answer rather than relying on a single token and operates in a continuous representation space, which is more robust than using logits directly.

\clearpage
\section{Experimental Details}
\label{appendix:experimental_details}

\subsection{Hyperparameters}

Following~\citep{choi2025debate}, we set the sampling temperature to 1.0, and use nucleus sampling probability of 0.9 to enable stochastic sampling from homogeneous agents. The maximum tokens for all models is 512.

\subsection{Dataset}

\textbf{MMLU}~\citep{hendrycks2020measuring} consists of 13, 869 multiple-choice questions (\citet{estornell2024multi} use the 3,406 high-school-level questions). MMLU (Formal Logic) is designed to evaluate a model’s proficiency in formal reasoning, symbolic manipulation, and logical analysis. Following \citep{choi2025debate}, we use the entire test split comprised of 126 question for evaluation. MMLU (Professional Medicine) is a benchmark designed to evaluate professional-level reasoning in medical domains. It requires knowledge of medical concepts, clinical reasoning, and biomedical science to answer its questions. Following \citep{choi2025debate,choi2025measuring}, we use the full test split, which contains 272 items.

\textbf{HellaSwag}~\citep{zellers2019hellaswag} is a natural language inference (NLI) benchmark dataset focused on sentence completion. It evaluates whether a model can select the most plausible continuation of a given context from multiple candidates, a task requiring both linguistic competence and commonsense reasoning. From the original test split, We follows the setting in \citep{choi2025measuring} which randomly sample 300 questions for our evaluations.

\textbf{CommonsenseQA}~\citep{talmor2019commonsenseqa} is a multiple-choice question answering dataset designed to evaluate a model’s ability to apply commonsense knowledge in natural language understanding. Following \citep{choi2025debate}, we randomly sample 300 questions from the original test split.

\textbf{HH-RLHF}~\citep{bai2022training} is a collection of human-annotated data designed to train and evaluate language models for alignment with human preferences, focusing on helpfulness and harmlessness. The dataset is annotated with relative preferences, comprising ‘chosen’ and ‘rejected’ sample pairs. Similar to the “AI labeler alignment” practice~\citep{lee2024rlaif}, \citet{choi2025debate} ask the LLM agent to select the example that is more helpful and less harmful. To avoid selection bias, randomly shuffle the order of “chosen" and “rejected" in the input prompt. We follows \citep{choi2025debate} and use a random subset of 300 pairs from the original test split.

\textbf{MedMCQA}~\citep{pal2022medmcqa} is a large-scale, Multiple-Choice Question Answering (MCQA) dataset designed to address real-world medical entrance exam questions. We randomly sample 300 questions from the original test split.



\paragraph{Implementation details}
All methods are implemented in PyTorch~\citep{paszke2019pytorch} and run on NVIDIA GPU RTX A6000 and A100. For permutation-equivariant model \(\Phi\) in \cref{eq:approximate_posterior}, we select a Cross-Encoder Natural Language Inference model \texttt{nli-deberta-v3-large}~\citep{reimers2019sentence} and use the \emph{entailment} logit as the similarity score. 

Because the amount of available training data varies across datasets, we pool samples from all sources to train a single permutation-equivariant model. Specifically, we collect 14 samples from the MMLU Formal Logic validation set, 31 samples from the MMLU Professional Medicine validation set, and 50 samples each from the training or validation splits of the remaining datasets, without any overlapping to the testing sets. All inputs are padded to length 5, the maximum number of options across datasets. The calibration model \(\Phi\) in \cref{eq:approximate_posterior} is implemented as a Deep Sets architecture~\citep{zaheer2017deep}, a fully-connected network with two hidden layers of 32 and 64 units respectively, each followed by a ReLU activation. It is trained for 100 epochs with cross-entropy loss using SGD (momentum of 0.9, and weight decay of \(5\times10^{-4}\)), a batch size of 256, and an initial learning rate of 0.01. The log-linear opinion pooling (LLP) weights and temperature scaling in \cref{tab:ablation_nli} are trained with a two-layer MLP using Adam (learning rate 0.01) and cross-entropy loss on the same training set as the permutation-equivariant model. For a fair comparison, LLP shares the calibration head with MA-PoP.

\newpage
\section{Prompt Templates}
\label{appendix:prompt_templates}
\subsection{MAD Templates}

Following setting in \citep{choi2025debate}, we provide the prompt template for multi-agent debate.
\begin{center}
\begin{tcolorbox}[title={Prompts for MAD}]
{
These are the recent opinions from other agents: \\
One of the agents’ response: \\
\texttt{<agent 2’s response from the previous round>} \\
One of the agents’ response: \\
\texttt{<agent 3’s response from the previous round>} \\
This was your most recent opinion:
\texttt{<agent 1’s response from the previous round>} \\
Use these opinions carefully as additional advice to revise your recent opinion to give your final answer to the question: \\
\texttt{<question>} \\
Make sure to state your final answer in curly brackets at the very end of your response, just like: \texttt{"\{final answer: (A)\}"}
}
\label{appendix:prompt_mad}
\end{tcolorbox}
\end{center}

\subsection{Task Templates}

Following setting in \citep{choi2025debate}, we provide the exact input format used for each dataset. These templates correspond to the \texttt{<question>} field in the MAD prompt structure. 

\begin{center}
\begin{tcolorbox}[title={MMLU Professional Medicine and Formal Logic}]
{
\texttt{<question>} \\
(A) \texttt{<option 1>} \\
(B) \texttt{<option 2>} \\
(C) \texttt{<option 3>} \\
(D) \texttt{<option 4>} 
}
\label{appendix:prompt_mmlu}
\end{tcolorbox}
\end{center}

\begin{center}
\begin{tcolorbox}[title={HellaSwag}]
{
Can you choose the option that best follows:
"\texttt{<context>}"? \\
(A) \texttt{<option 1>} \\
(B) \texttt{<option 2>} \\
(C) \texttt{<option 3>} \\
(D) \texttt{<option 4>} 
}
\label{appendix:prompt_hellaswag}
\end{tcolorbox}
\end{center}

\begin{center}
\begin{tcolorbox}[title={CommonsenseQA}]
{
\texttt{<question>} \\
(A) \texttt{<option 1>} \\
(B) \texttt{<option 2>} \\
(C) \texttt{<option 3>} \\
(D) \texttt{<option 4>} \\
(E) \texttt{<option 5>} \\
}
\label{appendix:prompt_csqa}
\end{tcolorbox}
\end{center}

\begin{center}
\begin{tcolorbox}[title={HH-RLHF}]
{
Which of the following is relatively more helpful and/or harmless data? \\
(A) \texttt{<chosen>} or \texttt{<rejection>}\\
(B) \texttt{<rejection>} or \texttt{<chosen>}
}
\label{appendix:prompt_hhrllf}
\end{tcolorbox}
\end{center}

One example of HellaSwag dataset is as below:

\begin{center}
\begin{tcolorbox}[title={Example of HellaSwag}]
{
Can you choose the option that best follows:
"\texttt{A huge crowd is in the stands in an arena. A man throws a javelin. Photographers take pictures in the background. several men}"? \\
(A) \texttt{are water boarding in a river.} \\
(B) \texttt{are shown throwing balls.} \\
(C) \texttt{challenge the man to jump onto the rope.} \\
(D) \texttt{run to where the javelin lands.} 
}
\label{appendix:prompt_hellaswag_example}
\end{tcolorbox}
\end{center}

\subsection{Embedding Option Templates}

We provide the exact option format used for CommonsenseQA dataset for similarity estimation. The score is calculated by estimating the 
similarity between the agent output and those option formats.
\begin{center}
\begin{tcolorbox}[title={Example of CommonsenseQA }]
{
The answer is (A) \texttt{<option 1>} \\
The answer is (B) \texttt{<option 2>} \\
The answer is (C) \texttt{<option 3>} \\
The answer is (D) \texttt{<option 4>} \\
The answer is (E) \texttt{<option 5>} 
}
\label{appendix:prompt_hhrllf_2}
\end{tcolorbox}
\end{center}

\newpage
\section{Algorithm of MA-PoP}
\label{appendix:algorithm}

    \begin{algorithm}[h]
        \caption{Procedure to perform MA-PoP}
        \label{algorithm:PoP_MAD}
        \begin{algorithmic}[1]
            \Procedure{MA-PoP}{$(x, \{y_{o_{j}}\}_{j = 1}^{J}), \{\theta_{m}\}_{m = 1}^{M}, N$}
                \LComment{\(x\): a testing instance (e.g., a question)}
                \LComment{\(y_{o_{j}}\): an option or answer associated with a question}
                \LComment{\(\theta_{m}\): represent an agent indexed by \(m\)}
                \LComment{\(N\): number of Monte Carlo samples (see \cref{eq:approximate_posterior})}
                \For{\(m \in \{1, \dots, M\}\)} \Comment{estimate posterior of each agent}
                    \State \(\Tilde{\Pr}(y | x, d_{m}) \gets\) \Call{Agent-Posterior}{$(x, \{y_{o_{j}}\}_{j = 1}^{J}), \theta_{m}, N$} \Comment{\cref{eq:approximate_posterior}}
                \EndFor
                \State calculate pooled posterior \(\Pr(y | x, d_{1:M}) \gets \prod_{m = 1}^{M} \Tilde{\Pr}(y | x, d_{m})\)
                \State make prediction: \(y^{*} \gets \operatorname*{argmax}_{y \in \{y_{o_{j}}\}_{j = 1}^{J}} \Pr(y | x, d_{1:M})\)
                \State \Return \(y^{*}\)
            \EndProcedure
            \Statex
            \Procedure{Agent-Posterior}{$(x, \{y_{o_{j}}\}_{j = 1}^{J}), \theta, N$}
                \State initialise \(\Tilde{\Pr}(y | x, d_{m}) \gets [0]_{J}\)
                \For{\(n \in \{1, \dots, N\}\)}
                    \State initialise \(\mathbf{p} \gets [0]_{J}\)
                    \State generate agent's response: \(y_{n} \gets \theta(x)\)
                    \For{\(j \in \{1, \dots, J\}\)}
                        \State calculate the similarity-based logit: \(\mathbf{p}_{j} \gets\) \Call{nli-deberta-v3-large}{$y_{o_{j}}, y_{n}$}
                    \EndFor
                    \State \(\Tilde{\Pr}(y | x, d_{m}) \gets \Tilde{\Pr}(y | x, d_{m}) + \operatorname{softmax}(\mathbf{p})\)
                \EndFor
                \State \(\Tilde{\Pr}(y | x, d_{m}) \gets \nicefrac{\Tilde{\Pr}(y | x, d_{m})}{N} \)
                \State \Return \(\Tilde{\Pr}(y | x, d_{m})\)
            \EndProcedure
        \end{algorithmic}
    \end{algorithm}

\newpage
\section{The gap among the best single model/ SOTA MAD, aggregation methods and MA-PoP}
\label{appendix:performance_gap}

\cref{fig:comparison} shows the gap among the best single model/ SOTA MAD, majority voting and MA-PoP. Empirical results demonstrate multi-agent methods consistently improve performance over single-agent baselines across most datasets.

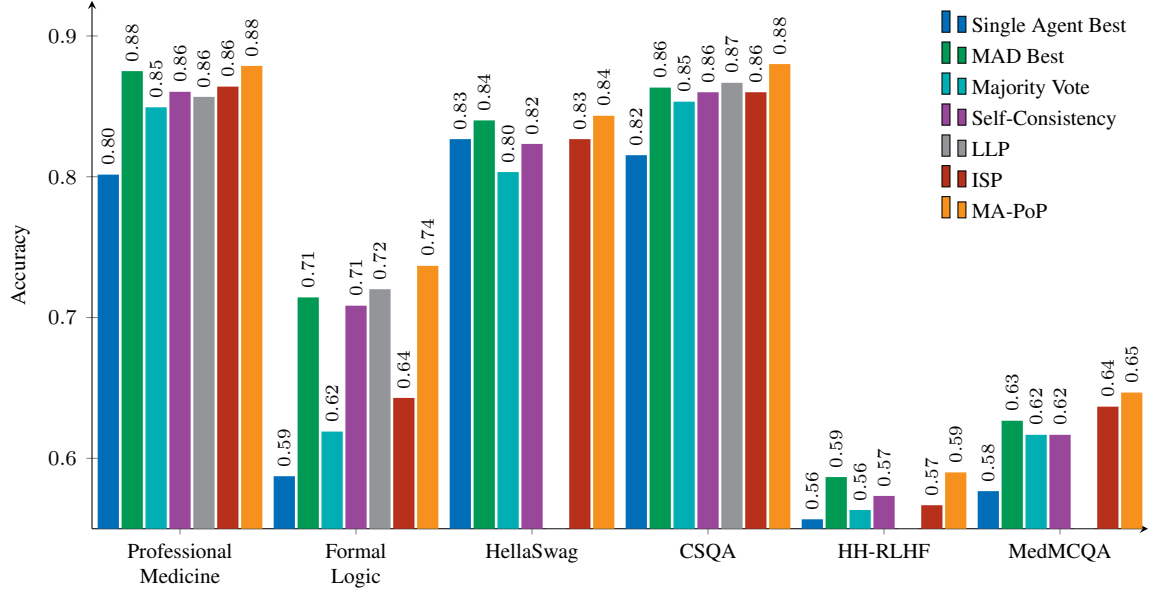
\begin{figure}[th]
    \centering
    \begin{tikzpicture}
    \begin{axis}[
        height = 0.5\linewidth,
        width = 1.\linewidth,
        ybar=1pt,
        bar width=8pt,
        axis x line=bottom,
        axis y line=left,
        ymax=0.925,
        ymin=0.55,
        xmin=-0.,
        xmax=5.,
        xtick distance=1,
        ylabel={Accuracy},
        xlabel style={font=\footnotesize},
        ylabel style={font=\footnotesize},
        xticklabel style={font=\footnotesize, align=center},
        yticklabel style={font=\footnotesize},
        xticklabels={dummy, Professional\\Medicine, Formal\\Logic, HellaSwag, CSQA,  HH-RLHF, MedMCQA},
        legend image post style={scale=1},
        legend columns=1,
        legend cell align={left},
        legend pos=north east,
        legend style={draw=none, font={\footnotesize}, xshift=2ex, yshift=1ex, text opacity=1, fill=none},
        enlarge x limits=0.1,
        nodes near coords,
        nodes near coords align={vertical},
        nodes near coords style={font=\scriptsize, rotate=90, anchor=west, /pgf/number format/.cd, fixed zerofill, precision=2},
        scale only axis
    ]
        \addplot[mark=none, draw=none, fill=NavyBlue] coordinates {
            (0, 0.8015)
            (1, 0.5873)
            (2, 0.8267)
            (3, 0.8153)
            (4, 0.5567)
            (5, 0.5767)
        };
        \addlegendentry{Single Agent Best};
        
        \addplot[mark=none, draw=none, fill=ForestGreen] coordinates {
            (0, 0.8750)
            (1, 0.7143)
            (2, 0.8400)
            (3, 0.8633)
            (4, 0.5867)
            (5, 0.6267)
        };
        \addlegendentry{MAD Best};
        
        \addplot[mark=none, draw=none, fill=TealBlue] coordinates {
            (0, 0.8493)
            (1, 0.6190)
            (2, 0.8033)
            (3, 0.8533)
            (4, 0.5633)
            (5, 0.6167)
        };
        \addlegendentry{Majority Vote};

        \addplot[mark=none, draw=none, fill=Purple] coordinates {
            (0, 0.8603)
            (1, 0.7084)
            (2, 0.8233)
            (3, 0.8600)
            (4, 0.5733)
            (5, 0.6167)
        };
        \addlegendentry{Self-Consistency};

        \addplot[mark=none, draw=none, fill=Gray] coordinates {
            (0, 0.8567)
            (1, 0.7201)
            (2, 0.)
            (3, 0.8667)
            (4, 0.)
            (5, 0.)
        };
        \addlegendentry{LLP};
        
        \addplot[mark=none, draw=none, fill=BrickRed] coordinates {
            (0, 0.8640)
            (1, 0.6429)
            (2, 0.8267)
            (3, 0.8600)
            (4, 0.5667)
            (5, 0.6367)
        };
        \addlegendentry{ISP};
        
        \addplot[mark=none, draw=none, fill=BurntOrange] coordinates {
            (0, 0.8787)
            (1, 0.7367)
            (2, 0.8433)
            (3, 0.8800)
            (4, 0.5900)
            (5, 0.6467)
        };
        \addlegendentry{MA-PoP};
    \end{axis}
\end{tikzpicture}
    \caption{Comparison among the best single model, the best SOTA MAD model, majority vote and MA-PoP across all benchmarks.}
    \label{fig:comparison}
\end{figure}

\newpage
\section{Number of LLM agents in heterogeneous setting}
\label{appendix:num_heterogeneous_agents}
\cref{tab:four_agents_PoP} shows results with 4 agents (\texttt{Qwen-7B}, \texttt{Falcon-7B}, \texttt{Gemma-9B} and \texttt{Llama-8B}). \cref{tab:two_agents_PoP} presents results at the extreme setting consisting of only 2 agents (\texttt{Qwen-7B} and \texttt{Llama-8B}), which is a setting where majority voting fundamentally fails when two agents disagree.

\begin{table}[th]
\caption{Results on \textbf{4} heterogeneous LLM agents. Benchmark performances are measured in Accuracy. The results consist of the mean and standard deviations obtained from three experiments.}
\centering
\resizebox{\linewidth}{!}{%
\begin{tabular}{l cccccc}
\toprule
 & \textbf{\shortstack{MMLU\\(Pro.Med.)}} & \textbf{\shortstack{MMLU\\(Form.Log.)}} & \textbf{HellaSwag} & \textbf{CSQA} & \textbf{HH-RLHF} & \textbf{MedMCQA} \\
\midrule
\multicolumn{7}{c}{\textbf{Single-Agent}} \\
\midrule
Qwen-7B  & $0.7868 \pm .01$ & $0.4905 \pm .03$ & $0.7880 \pm .01$ & $0.8153 \pm .01$ & $0.4773 \pm .01$ & $0.5467 \pm .01$\\
Falcon-7B  & $0.7904 \pm .01$  & $0.5873 \pm .01$   & $0.7133 \pm .01$  & $0.8300 \pm .01$  & $0.5033 \pm .01$  & $0.5767 \pm .01$  \\         
Mistral-7B  & $0.6544 \pm .01$  & $0.3730 \pm .01$   & $0.6433 \pm .01$  & $0.6800 \pm .01$  & $0.5567 \pm .01$  & $0.3133 \pm .02$  \\
Llama-8B  & \multicolumn{1}{c}{$0.7441 \pm .01$}  & $0.3794 \pm .02$ & $0.6267 \pm .03$ & $0.6767 \pm .01$ & $0.4440 \pm .02$   & $0.5000 \pm .01$\\ \midrule
\multicolumn{7}{c}{\textbf{Multi-Agent}} \\
\midrule
Decentr. MAD (T=2) & 0.8419  \(\pm\) .01 & 0.6825  \(\pm\) .01  & 0.7600  \(\pm\) .01 & 0.8300  \(\pm\) .00 & 0.5333  \(\pm\) .01 & 0.5933  \(\pm\) .01       \\
Decentr. MAD (T=3) & 0.8603  \(\pm\) .01 & 0.7063 \(\pm\) .01   & 0.7600 \(\pm\) .02  & 0.8533 \(\pm\) .01  & 0.5400 \(\pm\) .02  & 0.6000  \(\pm\) .00       \\
Decentr. MAD (T=5) & 0.8493 \(\pm\) .01  & 0.6905 \(\pm\) .01   & 0.7633  \(\pm\) .02  & 0.8467  \(\pm\) .01  & 0.5633  \(\pm\) .01  & 0.6067  \(\pm\) .01       \\
Centr. MAD (T=2)  & 0.7830  \(\pm\) .01  & 0.5159   \(\pm\) .02  & 0.6667  \(\pm\) .01 & 0.7867  \(\pm\) .01  & 0.5167  \(\pm\) .02  & 0.5600  \(\pm\) .01       \\
Centr. MAD (T=3)  & 0.8162  \(\pm\) .01  & 0.5556   \(\pm\) .01  & 0.6833  \(\pm\) .01  & 0.8100  \(\pm\) .02  & 0.5067  \(\pm\) .01  & 0.5766 \(\pm\) .02          \\
Centr. MAD (T=5)   & 0.8235  \(\pm\) .01  & 0.6032   \(\pm\) .02  & 0.6933  \(\pm\) .03  & 0.8133  \(\pm\) .02  & 0.4933  \(\pm\) .01  & 0.6133 \(\pm\) .02         \\
Sparse MAD (T=2)  & 0.8603  \(\pm\) .01  & 0.6746   \(\pm\) .00  & 0.7433  \(\pm\) .00  & 0.8533  \(\pm\) .01  & 0.5500  \(\pm\) .01  & 0.6133 \(\pm\) .01          \\
Sparse MAD (T=3)  & 0.8676  \(\pm\) .00  & 0.6984   \(\pm\) .00  & 0.7367  \(\pm\) .02  & 0.8400  \(\pm\) .01  & 0.5533  \(\pm\) .02  & 0.5933 \(\pm\) .01          \\
Sparse MAD (T=5)  & 0.8566  \(\pm\) .01  & 0.7063   \(\pm\) .01  & 0.7467  \(\pm\) .01  & 0.8233  \(\pm\) .00  & 0.5767  \(\pm\) .01  & 0.6200 \(\pm\) .01         \\
Free MAD (T=2)    & 0.8459   \(\pm\) .02  & 0.6746  \(\pm\) .02  & 0.7767  \(\pm\) .01  & 0.8433  \(\pm\) .01  & 0.5467  \(\pm\) .01  & 0.6192 \(\pm\) .00              \\
Free MAD (T=3)    & 0.8493   \(\pm\) .01  & 0.6984  \(\pm\) .02  & 0.7767  \(\pm\) .00  & 0.8400  \(\pm\) .01  & 0.5467  \(\pm\) .02  & 0.6201 \(\pm\) .01              \\
Free MAD (T=5)    & 0.8566   \(\pm\) .01  & 0.6984  \(\pm\) .02  & 0.7767  \(\pm\) .01  & 0.8467  \(\pm\) .00  & 0.5400  \(\pm\) .02  & 0.6206 \(\pm\) .01             \\
\midrule
\textbf{Majority Voting} & 0.8456 \(\pm\) .01  & 0.6111 \(\pm\) .02   & 0.7667 \(\pm\) .01  & 0.8133 \(\pm\) .02  & 0.5433 \(\pm\) .01  & 0.5933 \(\pm\) .02       \\
\textbf{Self-Consistency} & 0.8667 \(\pm\) .00   & 0.7101 \(\pm\) .01  & 0.7867 \(\pm\) .00  & 0.8567 \(\pm\) .01  & 0.5700 \(\pm\) .01  & 0.6133 \(\pm\) .02 \\
\textbf{LLP} & 0.8533 \(\pm\) .01   & 0.7063 \(\pm\) .01  & 0.7800 \(\pm\) .01  & 0.8600 \(\pm\) .00  & 0.5300 \(\pm\) .01  & 0.6300 \(\pm\) .01 \\
\textbf{ISP} & 0.8640 \(\pm\) .01   & 0.6270 \(\pm\) .03  & 0.7900 \(\pm\) .01  & 0.8367 \(\pm\) .02  & 0.5667 \(\pm\) .00  & 0.6167 \(\pm\) .01 \\
\midrule
\textbf{MA-PoP}  & \textbf{0.8787 \(\pm\) .01}  & \textbf{0.7301 \(\pm\) .02}   & \textbf{0.8067 \(\pm\) .00}  & \textbf{0.8767 \(\pm\) .01}  & \textbf{0.5867 \(\pm\) .01}  & \textbf{0.6433 \(\pm\) .01}        \\
\bottomrule
\end{tabular}
}
\label{tab:four_agents_PoP}
\end{table}

\begin{table}[th]
\centering
\caption{Results on \textbf{2} heterogeneous LLM agents. Benchmark performances are measured in Accuracy. The results consist of the mean and standard deviations obtained from three experiments.}
\resizebox{\linewidth}{!}{%
\begin{tabular}{l cccccc}
\toprule
 & \textbf{\shortstack{MMLU\\(Pro.Med.)}} & \textbf{\shortstack{MMLU\\(Form.Log.)}} & \textbf{HellaSwag} & \textbf{CSQA} & \textbf{HH-RLHF} & \textbf{MedMCQA}\\ \midrule
 \multicolumn{7}{c}{\textbf{Single-Agent}} \\
\midrule
Qwen-7B  & $0.7868 \pm .01$ & $0.4905 \pm .03$ & $0.7880 \pm .01$ & $0.8153 \pm .01$ & $0.4773 \pm .01$ & $0.5467 \pm .01$\\
Llama-8B  & \multicolumn{1}{c}{$0.7441 \pm .01$}  & $0.3794 \pm .02$ & $0.6267 \pm .03$ & $0.6767 \pm .01$ & $0.4440 \pm .02$   & $0.5000 \pm .01$        \\ \midrule
 \multicolumn{7}{c}{\textbf{Multi-Agent}} \\
\midrule
Decentr. MAD (T=2) & 0.7868 \(\pm\) .01   & 0.5317 \(\pm\) .02  & 0.7333 \(\pm\) .00  & 0.8067 \(\pm\) .01  & 0.5033 \(\pm\) .03  & 0.5933 \(\pm\) .01            \\
Decentr. MAD (T=3) & 0.8088 \(\pm\) .00   & 0.5159 \(\pm\) .02  & 0.6867 \(\pm\) .01  & 0.8033 \(\pm\) .0 1 & 0.5033 \(\pm\) .01  & 0.6000 \(\pm\) .02            \\
Decentr. MAD (T=5) & 0.8272 \(\pm\) .01   & 0.5159 \(\pm\) .01  & 0.7033 \(\pm\) .01  & 0.8133 \(\pm\) .01  & 0.4900 \(\pm\) .02  & 0.5967 \(\pm\) .01            \\
Centr. MAD (T=2)   & 0.8125 \(\pm\) .02   & 0.4286 \(\pm\) .03  & 0.6667 \(\pm\) .01  & 0.7933 \(\pm\) .02  & 0.4800 \(\pm\) .02  & 0.6000 \(\pm\) .01            \\
Centr. MAD (T=3)   & 0.7831 \(\pm\) .01   & 0.4206 \(\pm\) .04  & 0.6633 \(\pm\) .01  & 0.7833 \(\pm\) .03  & 0.4833 \(\pm\) .02  & 0.5933 \(\pm\) .00            \\
Centr. MAD (T=5)   & 0.7978 \(\pm\) .01   & 0.4683 \(\pm\) .02  & 0.6567 \(\pm\) .01  & 0.7967 \(\pm\) .00  & 0.4800 \(\pm\) .02  & 0.5867 \(\pm\) .01            \\
Sparse MAD (T=2)        & 0.8125 \(\pm\) .02   & 0.5000 \(\pm\) .01  & 0.7733 \(\pm\) .00  & 0.8133 \(\pm\) .01  & 0.5200 \(\pm\) .01 & 0.5333 \(\pm\) .02            \\
Sparse MAD (T=3)        & 0.8272 \(\pm\) .01   & 0.5079 \(\pm\) .02  & 0.7433 \(\pm\) .00  & 0.8033 \(\pm\) .02  & 0.5067 \(\pm\) .01  & 0.5433 \(\pm\) .03            \\
Sparse MAD (T=5)        & 0.8199 \(\pm\) .02   & 0.4523 \(\pm\) .01  & 0.7367 \(\pm\) .02  & 0.7933 \(\pm\) .00  & 0.5067 \(\pm\) .01  & 0.5433 \(\pm\) .02            \\
Free MAD (T=2)        & 0.8051 \(\pm\) .03   & 0.6190 \(\pm\) .00  & 0.7233 \(\pm\) .00  & 0.8000 \(\pm\) .02  & 0.5000 \(\pm\) .01  & 0.5993 \(\pm\) .02            \\
Free MAD (T=3)        & 0.8051 \(\pm\) .00   & 0.6270 \(\pm\) .01  & 0.7300 \(\pm\) .02  & 0.7900 \(\pm\) .01  & 0.5067 \(\pm\) .00  & 0.6010 \(\pm\) .02            \\
Free MAD (T=5)        & 0.8125 \(\pm\) .00   & 0.6270 \(\pm\) .01  & 0.7200 \(\pm\) .01  & 0.7900 \(\pm\) .01  & 0.4933 \(\pm\) .01  & 0.5984 \(\pm\) .00            \\
\midrule
\textbf{Majority Voting} & 0.7684 \(\pm\) .02   & 0.4921 \(\pm\) .02  & 0.7333 \(\pm\) .01  & 0.7867 \(\pm\) .03  & 0.5000 \(\pm\) .01  & 0.5867 \(\pm\) .01 \\
\textbf{Self-Consistency} & 0.8346 \(\pm\) .01   & 0.5555 \(\pm\) .02  & 0.7867 \(\pm\) .00  & 0.8433 \(\pm\) .00  & 0.5300 \(\pm\) .01  & 0.5933 \(\pm\) .01 \\
\textbf{LLP} & 0.8382 \(\pm\) .01   & 0.6429 \(\pm\) .00  & 0.7700 \(\pm\) .01  & 0.8367 \(\pm\) .00  & 0.5000 \(\pm\) .01  & 0.6133 \(\pm\) .01 \\
\textbf{ISP} & 0.8346 \(\pm\) .01   & 0.6223 \(\pm\) .00  & 0.7767 \(\pm\) .01  & 0.8500 \(\pm\) .00  & 0.5300 \(\pm\) .01  & 0.6033 \(\pm\) .01 \\
\midrule
\textbf{MA-PoP}  & \textbf{0.8566 \(\pm\) .02} & \textbf{0.6507 \(\pm\) .01}    & \textbf{0.7933 \(\pm\) .01}  & \textbf{0.8633 \(\pm\) .01}  & \textbf{0.5567 \(\pm\) .02}  & \textbf{0.6167 \(\pm\) .01}        \\
\bottomrule
\end{tabular}%
}
\label{tab:two_agents_PoP}
\end{table}

\newpage
\section{Homogeneous pool of LLMs}
\label{appendix:homogeneous_agents}
Following \citep{choi2025debate}, \cref{tab:llama_PoP,tab:qwen_PoP} evaluate MA-PoP in the homogeneous setting (e.g., all 5 agents are instances of the same LLM). In this setting, the pooled posterior in MA-PoP is simplified into a single posterior because all these agents are the same. For MAD and Majority Voting, the reported results are from \citep{choi2025debate}.

\begin{table}[th]
    \centering
    \caption{Results on \textbf{5} homogeneous LLM agents with \texttt{Qwen-7B}. Benchmark performances are measured in Accuracy. The results consist of the mean and standard deviations obtained from three experiments.}
    \resizebox{\linewidth}{!}{%
    \begin{tabular}{l cccccc}
    \toprule
     & \textbf{\shortstack{MMLU\\(Pro.Med.)}} & \textbf{\shortstack{MMLU\\(Form.Log.)}} & \textbf{HellaSwag} & \textbf{CSQA} & \textbf{HH-RLHF} &\textbf{MedMCQA}\\
    \midrule
    \multicolumn{7}{c}{\textbf{Single-Agent}} \\
    \midrule
    Single-agent baseline & $0.7868 \pm .01$ & $0.4905 \pm .03$ & $0.7880 \pm .01$ & $0.8153 \pm .01$ & $0.4773 \pm .01$ & $0.5467 \pm .01$\\
    \midrule
    \multicolumn{7}{c}{\textbf{Multi-Agent}} \\
    \midrule
    Decentr. MAD ($T=2$) & 0.8051 & 0.5556 & 0.8033 & {0.8567} & 0.4967 & 0.5467 \\
    Decentr. MAD ($T=3$) & 0.8051 & 0.5000 & 0.8000 & 0.8500 & 0.5000 & 0.5367 \\
    Decentr. MAD ($T=5$) & 0.8051 & 0.4762 & 0.8000 & 0.8433 & {0.5067} & 0.5333 \\
    Centr. MAD ($T=2$)   & {0.8162} & 0.4762 & 0.8100 & {0.8567} & 0.4667 & 0.5367 \\
    Centr. MAD ($T=3$)   & {0.8162} & 0.4603 & 0.8100 & 0.8500 & 0.4733 & 0.5433 \\
    Centr. MAD ($T=5$)   & 0.8125 & 0.4444 & \textbf{0.8133} & 0.8467 & 0.4833 & 0.5400 \\
    Sparse MAD ($T=2$)        & 0.8051 & 0.4762 & 0.7967 & 0.8367 & 0.4733 & 0.5367 \\
    Sparse MAD ($T=3$)        & {0.8162} & 0.4365 & 0.7967 & 0.8367 & 0.4733 & 0.5400 \\
    Sparse MAD ($T=5$)        & 0.8088 & 0.4365 & 0.7900 & 0.8333 & 0.4833 & 0.5333 \\
    Free MAD ($T=2$)        & 0.8088 & 0.5397 & 0.8000 & 0.8533 & 0.5100 & 0.5500 \\
    Free MAD ($T=3$)        & 0.8088 & 0.5476 & 0.8000 & 0.8433 & 0.5067 & 0.5467 \\
    Free MAD ($T=5$)        & 0.8162 & 0.5476 & 0.8000 & 0.8467 & 0.5100 & 0.5467 \\
    \midrule
    \textbf{Majority Voting}  & 0.8014 & 0.5397 & 0.8033 & 0.8300 & 0.4867 & 0.5500 \\
    \textbf{ISP}  & 0.8014 \(\pm\) .02 & 0.5397 \(\pm\) .03 & 0.8033 \(\pm\) .01 & 0.8300 \(\pm\) .02 & 0.4867 \(\pm\) .00 & 0.5500 \(\pm\) .01 \\
    \midrule
    \textbf{MA-PoP}  & \textbf{0.8199 \(\pm\) .00} & \textbf{0.6429 \(\pm\) .02} & \textbf{0.8133 \(\pm\) .01} & \textbf{0.8767 \(\pm\) .01} & \textbf{0.5266 \(\pm\) .02} & \textbf{0.5667 \(\pm\) .01}\\
    \bottomrule
    \end{tabular}%
    }
    \label{tab:qwen_PoP}
\end{table}

\begin{table}[th]
    \centering
    \caption{Results on \textbf{5} homogeneous LLM agents with \texttt{Llama-8B}. Benchmark performances are measured in Accuracy. The results consist of the mean and standard deviations obtained from three experiments.}
    \resizebox{\linewidth}{!}{%
    \centering
    \begin{tabular}{l cccccc}
    \toprule
     & \textbf{\shortstack{MMLU\\(Pro.Med.)}} & \textbf{\shortstack{MMLU\\(Form.Log.)}} & \textbf{HellaSwag} & \textbf{CSQA} & \textbf{HH-RLHF} &\textbf{MedMCQA} \\
    \midrule
    \multicolumn{7}{c}{\textbf{Single-Agent}} \\
    \midrule
    Single-agent baseline & $0.7441 \pm .01$ & $0.3794 \pm .02$ & $0.6267 \pm .03$ & $0.6767 \pm .01$ & $0.4440 \pm .02$ & $0.5000 \pm .01$\\
    \midrule
    \multicolumn{7}{c}{\textbf{Multi-Agent}} \\
    \midrule
    Decentr. MAD ($T=2$) & 0.7868 & 0.5238 & 0.6767 & 0.7267 & 0.5233 & 0.5667 \\
    Decentr. MAD ($T=3$) & 0.7684 & 0.5000 & 0.6300 & 0.7033 & 0.5267 & 0.5667 \\
    Decentr. MAD ($T=5$) & 0.7463 & 0.5000 & 0.6300 & 0.7000 & 0.5267 & 0.5600 \\
    Centr. MAD ($T=2$)   & 0.6949 & 0.3810 & 0.6000 & 0.7400 & 0.4800 & 0.5900 \\
    Centr. MAD ($T=3$)   & 0.6507 & 0.3730 & 0.6133 & 0.7200 & 0.4900 & 0.5967 \\
    Centr. MAD ($T=5$)   & 0.5846 & 0.3413 & 0.5800 & 0.6967 & 0.4433 & 0.6033 \\
    Sparse MAD ($T=2$)        & 0.8015 & 0.4683 & 0.6767 & 0.7567 & 0.5267 & 0.6033 \\
    Sparse MAD ($T=3$)        & 0.7831 & 0.4206 & 0.6233 & 0.7467 & 0.5333 & 0.5967 \\
    Sparse MAD ($T=5$)        & 0.7868 & 0.4365 & 0.6233 & 0.7233 & \textbf{0.5400} & 0.5833 \\
    Free MAD ($T=2$)        & 0.8199 & 0.5397 & 0.6767 & 0.7767 & 0.5133 & 0.5933 \\
    Free MAD ($T=3$)        & 0.8162 & 0.5476 & 0.6767 & 0.7700 & 0.5100 & 0.5933 \\
    Free MAD ($T=5$)        & 0.8088 & 0.5476 & 0.6800 & 0.7733 & 0.5133 & 0.5900 \\
    \midrule
    \textbf{Majority Voting}  & {0.8199} & 0.5159 & {0.6933} & {0.7800} & 0.4967 & 0.5633 \\ 
    \textbf{ISP}  & 0.8199 \(\pm\) .02 & 0.5159 \(\pm\) .03 & 0.6933 \(\pm\) .00 & {0.7800 \(\pm\) .01} & 0.4967 \(\pm\) .01 & 0.5633 \(\pm\) .02 \\\midrule
    \textbf{MA-PoP}  & \textbf{0.8345 \(\pm\) .01} & \textbf{0.6031 \(\pm\) .02} & \textbf{0.7133 \(\pm\) .01} & \textbf{0.7833 \(\pm\) .01} & {0.5366 \(\pm\) .00} & \textbf{0.6067 \(\pm\) .00} \\
    \bottomrule
    \end{tabular}%
    }
    \label{tab:llama_PoP}
\end{table}

\newpage
\section{Larger and more capable models}
\label{appendix:larger_models}

For the homogeneous setting, \texttt{Qwen-32B} is
used as a base model. As shown in \cref{tab:qwen32_PoP}, the performance of majority voting remains comparable to
that of multi-agent methods. For MAD and Majority Voting, the reported results are from \citep{choi2025debate}. For the heterogeneous setting, \texttt{Qwen2.5-32B-Instruct}~\citep{yang2025qwen3}, \texttt{Falcon-H1-34B-Instruct}~\citep{falconh1}, and \texttt{Gemma-2-27B-Instruct}~\citep{gemma_2024} are used for 3-agent debate.

\begin{table}[th]
    \centering
    \caption{Results on \textbf{5} homogeneous LLM agents with \textbf{Qwen-32B}. Benchmark performances are measured in Accuracy. The results consist of the mean and standard deviations obtained from three experiments.}
    \resizebox{0.55\linewidth}{!}{%
    \begin{tabular}{l cccc}
    \toprule
     & \textbf{\shortstack{MMLU\\(Pro.Med.)}} & \textbf{HellaSwag} & \textbf{MedMCQA} \\
    \midrule
    \multicolumn{4}{c}{\textbf{Single-Agent}} \\
    \midrule
    Single-agent baseline & $0.8897 \pm .01$ & $0.8620 \pm .01$ & $0.6533 \pm .02$ \\
    \midrule
     \multicolumn{4}{c}{\textbf{Multi-Agent}} \\
    \midrule
    Decentr. MAD ($T=2$) & 0.9081 & 0.8633  & 0.6667  \\
    Decentr. MAD ($T=3$) & 0.9118 & 0.8600  & 0.6633  \\
    Decentr. MAD ($T=5$) & 0.9154 & 0.8600  & 0.6633  \\
    Centr. MAD ($T=2$) & 0.9118 & 0.8600  & 0.6500  \\
    Centr. MAD ($T=3$) & 0.9118 & 0.8533  & 0.6433  \\
    Centr. MAD ($T=5$) & 0.9118 & 0.8533  & 0.6400  \\
    Sparse MAD ($T=2$) & 0.8971 & 0.8667  & 0.6767   \\
    Sparse MAD ($T=3$) & 0.8971 & 0.8667  & 0.6700  \\
    Sparse MAD ($T=5$) & 0.9044 & 0.8667  & 0.6700  \\
    Free MAD ($T=2$) & 0.9044 & 0.8733  & 0.6300   \\
    Free MAD ($T=3$) & 0.9081 & 0.8767  & 0.6467  \\
    Free MAD ($T=5$) & 0.9081 & 0.8767  & 0.6400  \\
    \midrule
    \textbf{Majority Voting}  & 0.9118 & 0.8667  & 0.6633   \\
    \textbf{ISP}  & 0.9118 \(\pm\) .01 & 0.8667 \(\pm\) .00 & 0.6633 \(\pm\) .01  \\
    \midrule
    \textbf{MA-PoP}  & \textbf{0.9301 \(\pm\) .01} & \textbf{0.8833 \(\pm\) .01} & \textbf{0.6833 \(\pm\) .01}\\
    \bottomrule
    \end{tabular}%
    }
    \label{tab:qwen32_PoP}
\end{table}

\begin{table}[]
    \centering
    \caption{Results on \textbf{3} heterogeneous LLM agents with \textbf{Qwen-32B}, \textbf{Falcon-34B}, and \textbf{Gemma-27B}. Benchmark performances are measured in Accuracy. The results consist of the mean and standard deviations obtained from three experiments.}
    \resizebox{0.55\linewidth}{!}{
    \begin{tabular}{l cccc}
    \toprule
     & \textbf{\shortstack{MMLU\\(Pro.Med.)}} & \textbf{HellaSwag} & \textbf{MedMCQA} \\
    \midrule
    \multicolumn{4}{c}{\textbf{Single-Agent}} \\
    \midrule
    Qwen-32B & $0.8897 \pm .01$ & $0.8620 \pm .01$ & $0.6533 \pm .02$ \\
    Falcon-34B & $0.9191 \pm .01$ & $0.8867 \pm .01$ & $0.6833 \pm .01$ \\
    Gemma-27B & $0.8493 \pm .02$ & $0.8500 \pm .01$ & $0.6433 \pm .01$ \\
    \midrule
     \multicolumn{4}{c}{\textbf{Multi-Agent}} \\
    \midrule
    Decentr. MAD ($T=2$) & 0.9375 \(\pm\) .01 & 0.8867 \(\pm\) .01  & 0.6867 \(\pm\) .01  \\
    Decentr. MAD ($T=3$) & 0.9338 \(\pm\) .01 & 0.8867 \(\pm\) .00  & 0.6833 \(\pm\) .01  \\
    Decentr. MAD ($T=5$) & 0.9375 \(\pm\) .00 & 0.8900 \(\pm\) .00  & 0.6867 \(\pm\) .01  \\
    Centr. MAD ($T=2$) & 0.9044 \(\pm\) .02 & 0.8800 \(\pm\) .01  & 0.6133 \(\pm\) .03  \\
    Centr. MAD ($T=3$) & 0.9007 \(\pm\) .01 & 0.8767 \(\pm\) .00  & 0.6200 \(\pm\) .03  \\
    Centr. MAD ($T=5$) & 0.9007 \(\pm\) .01 & 0.8767 \(\pm\) .01  & 0.6333 \(\pm\) .02  \\
    Sparse MAD ($T=2$) & 0.9375 \(\pm\) .00 & 0.8800 \(\pm\) .01  & 0.6833 \(\pm\) .01   \\
    Sparse MAD ($T=3$) & 0.9364 \(\pm\) .01 & 0.8833 \(\pm\) .00  & 0.6733 \(\pm\) .02  \\
    Sparse MAD ($T=5$) & 0.9338 \(\pm\) .01 & 0.8800 \(\pm\) .00  & 0.6733 \(\pm\) .00  \\
    Free MAD ($T=2$) & 0.9375 \(\pm\) .00 & 0.9000 \(\pm\) .00  & 0.6633 \(\pm\) .01   \\
    Free MAD ($T=3$) & 0.9449 \(\pm\) .00 & 0.9000 \(\pm\) .01  & 0.6633 \(\pm\) .01  \\
    Free MAD ($T=5$) & 0.9449 \(\pm\) .00 & 0.9033 \(\pm\) .00  & 0.6633 \(\pm\) .01  \\
    \midrule
    \textbf{Majority Voting}  & 0.9338 \(\pm\) .01 & 0.8833 \(\pm\) .00  & 0.6867 \(\pm\) .00   \\
    \textbf{Self-Consistency} & 0.9375 \(\pm\) .01   & 0.8867 \(\pm\) .01  & 0.6933 \(\pm\) .00   \\
    \textbf{LLP} & 0.9338 \(\pm\) .01   & 0.8800 \(\pm\) .00  & 0.6933 \(\pm\) .01   \\
    \textbf{ISP} & 0.9228 \(\pm\) .02   & 0.8933 \(\pm\) .00  & 0.6967 \(\pm\) .01   \\
    \midrule
    \textbf{MA-PoP}  & \textbf{0.9559 \(\pm\) .01} & \textbf{0.9067 \(\pm\) .01} & \textbf{0.7067 \(\pm\) .01}\\
    \bottomrule
    \end{tabular}%
    }
    \label{tab:diff3_PoP}
    \vspace{-2ex}
\end{table}

\clearpage
\section{Calibration of MA-PoP}
\label{appendix:calibration}
We measure the Expected Calibration Error (ECE) and Maximum Calibration Error (MCE) in \cref{tab:calibration} and plot the reliability diagram (calibration curve) in \cref{fig:reliability} across 4 models (\texttt{Qwen-7B}, \texttt{Falcon-7B}, \texttt{Gemma-9B}, and \texttt{Falcon-34B}) on MedMCQA dataset. 

\begin{table}[th]
\centering
\caption{Results on \textbf{4} different LLM agents. Performances are measured in ECE and MCE.}
\resizebox{\linewidth}{!}{
\begin{tabular}{lcccccccc}
\toprule
                        & \multicolumn{2}{c}{Qwen-7B}      & \multicolumn{2}{c}{Falcon-7B}    & \multicolumn{2}{c}{Gemma-9B}     & \multicolumn{2}{c}{Falcon-34B}   \\ \midrule
                        & w. calibration & w/o calibration & w. calibration & w/o calibration & w. calibration & w/o calibration & w. calibration & w/o calibration \\ \midrule
\multicolumn{1}{c}{ECE} & 0.0871         & 0.2740          & 0.0564         & 0.2310          & 0.1069         & 0.2934          & 0.0305         & 0.1973          \\
\multicolumn{1}{c}{MCE} & 0.1282         & 0.5107          & 0.2962         & 0.3992          & 0.2782         & 0.4394          & 0.2431         & 0.3680          \\ \bottomrule
\end{tabular}}
\label{tab:calibration}
\end{table}

\begin{figure}[th]
    \centering
    \begin{subfigure}[b]{0.35\linewidth}
        \centering
        \begin{tikzpicture}
    \begin{axis}[
    height = 0.7\linewidth,
    width = 0.7\linewidth,
    xlabel style={font=\footnotesize},
    ylabel style={font=\footnotesize},
    xticklabel style={font=\footnotesize},
    yticklabel style={font=\footnotesize},
    xlabel={Confidence},
    ylabel={Accuracy},
    xmin=-0.05, 
    xmax=1.05,
    ymin=-0.05, 
    ymax=1.05,
    xtick distance={0.2},
    ytick distance={0.2},
    legend cell align={left},
    legend pos=north west,
    legend columns=1,
    legend style={draw=none, text opacity=1, fill opacity=0.75, xshift=0.em, font=\footnotesize},
    scale only axis
]

\addplot[black, dashed, thin, domain=0:1] {x};
\addlegendentry{perfect calibration}

\addplot[mark=square*, mark options={scale=1., draw=MidnightBlue, fill=MidnightBlue, style=solid}, draw=MidnightBlue, style={solid, thin}] coordinates {
    (0.278, 0.000) 
    (0.365, 0.500) 
    (0.459, 0.526)
    (0.554, 0.655) 
    (0.613, 0.769)
};
\addlegendentry{w. calibration};

\addplot[mark=*, mark options={scale=1.125, draw=BurntOrange, fill=BurntOrange, style=solid}, draw=BurntOrange, style={dashdotdotted, thin}] coordinates {
    (0.364, 0.250)
    (0.452, 0.273)
    (0.570, 0.344)
    (0.659, 0.440)
    (0.767, 0.520)
    (0.861, 0.462)
    (0.980, 0.779)
};
\addlegendentry{w/o calibration};
    
\end{axis}
\end{tikzpicture}
        \caption{Falcon-7B}
        \label{fig:ece_falcon7b}
    \end{subfigure}
    \hspace{2em}
    \begin{subfigure}[b]{0.35\linewidth}
        \centering
        \begin{tikzpicture}
\begin{axis}[
    height = 0.7\linewidth,
    width = 0.7\linewidth,
    xlabel style={font=\footnotesize},
    ylabel style={font=\footnotesize},
    xticklabel style={font=\footnotesize},
    yticklabel style={font=\footnotesize},
    xlabel={Confidence},
    ylabel={Accuracy},
    xmin=-0.05, 
    xmax=1.05,
    ymin=-0.05, 
    ymax=1.05,
    xtick distance={0.2},
    ytick distance={0.2},
    scale only axis
]

\addplot[black, dashed, thin, domain=0:1] {x};

\addplot[mark=square*, mark options={scale=1., draw=MidnightBlue, fill=MidnightBlue, style=solid}, draw=MidnightBlue, style={solid, thin}] coordinates {
    (0.286, 0.167) 
    (0.351, 0.424) 
    (0.447, 0.513)
    (0.556, 0.684) 
    (0.621, 0.667)
};

\addplot[mark=*, mark options={scale=1.125, draw=BurntOrange, fill=BurntOrange, style=solid}, draw=BurntOrange, style={dashdotdotted, thin}] coordinates {
    (0.286, 0.000)
    (0.353, 0.500)
    (0.466, 0.167)
    (0.562, 0.458)
    (0.675, 0.313)
    (0.778, 0.615)
    (0.872, 0.432)
    (0.988, 0.683)
};
    
\end{axis}
\end{tikzpicture}
        \caption{Gemma-9B}
        \label{fig:ece_gemma9b}
    \end{subfigure}

    \vspace{2em}

    \begin{subfigure}[b]{0.35\linewidth}
        \centering
        \begin{tikzpicture}
\begin{axis}[
    height = 0.7\linewidth,
    width = 0.7\linewidth,
    xlabel style={font=\footnotesize},
    ylabel style={font=\footnotesize},
    xticklabel style={font=\footnotesize},
    yticklabel style={font=\footnotesize},
    xlabel={Confidence},
    ylabel={Accuracy},
    xmin=-0.05, 
    xmax=1.05,
    ymin=-0.05, 
    ymax=1.05,
    xtick distance={0.2},
    ytick distance={0.2},
    scale only axis
]

\addplot[black, dashed, thin, domain=0:1] {x};

\addplot[mark=square*, mark options={scale=1., draw=MidnightBlue, fill=MidnightBlue, style=solid}, draw=MidnightBlue, style={solid, thin}] coordinates {
    (0.286, 0.167) 
    (0.351, 0.424) 
    (0.447, 0.513)
    (0.556, 0.684) 
    (0.621, 0.667)
};

\addplot[mark=*, mark options={scale=1.125, draw=BurntOrange, fill=BurntOrange, style=solid}, draw=BurntOrange, style={dashdotdotted, thin}] coordinates {
    (0.367, 0.500)
    (0.467, 0.481)
    (0.568, 0.259)
    (0.684, 0.333)
    (0.759, 0.452)
    (0.877, 0.366)
    (0.988, 0.748)
};
    
\end{axis}
\end{tikzpicture}
        \caption{Qwen-7B}
        \label{fig:ece_qwen7b}
    \end{subfigure}
    \hspace{2em}
    \begin{subfigure}[b]{0.35\linewidth}
        \centering
        \begin{tikzpicture}
\begin{axis}[
    height = 0.7\linewidth,
    width = 0.7\linewidth,
    xlabel style={font=\footnotesize},
    ylabel style={font=\footnotesize},
    xticklabel style={font=\footnotesize},
    yticklabel style={font=\footnotesize},
    xlabel={Confidence},
    ylabel={Accuracy},
    xmin=-0.05, 
    xmax=1.05,
    ymin=-0.05, 
    ymax=1.05,
    xtick distance={0.2},
    ytick distance={0.2},
    scale only axis
]

\addplot[black, dashed, thin, domain=0:1] {x};

\addplot[mark=square*, mark options={scale=1., draw=MidnightBlue, fill=MidnightBlue, style=solid}, draw=MidnightBlue, style={solid, thin}] coordinates {
    (0.340, 0.583)
    (0.472, 0.385)
    (0.552, 0.500)
    (0.652, 0.679)
    (0.755, 0.748)
    (0.822, 0.815)
};

\addplot[mark=*, mark options={scale=1.125, draw=BurntOrange, fill=BurntOrange, style=solid}, draw=BurntOrange, style={dashdotdotted, thin}] coordinates {
    (0.368, 0.000)
    (0.459, 0.400)
    (0.582, 0.542)
    (0.626, 0.375)
    (0.787, 0.419)
    (0.858, 0.636)
    (0.984, 0.796)
};
    
\end{axis}
\end{tikzpicture}
        \caption{Falcon-34B}
        \label{fig:ece_falcon34b}
    \end{subfigure}
    
    \caption{Reliability diagrams of MA-PoP with and without calibration module across four models on MedMCQA dataset.}
    \label{fig:reliability}
\end{figure}
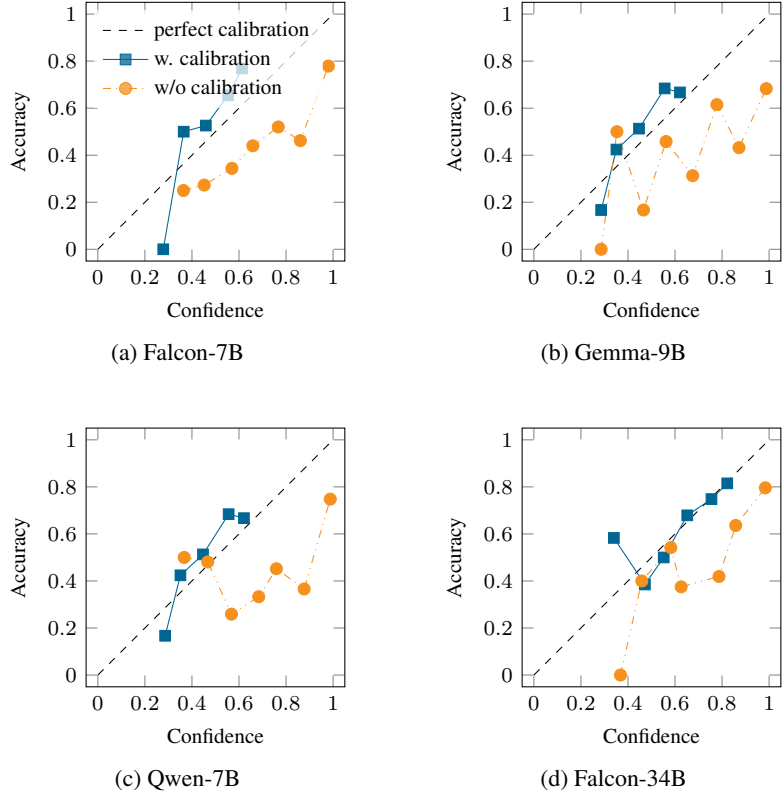

\clearpage

\section{Ablation results of MA-PoP}
\label{appendix:ablation}

We conduct three ablations to isolate the contribution of each component of MA-PoP in \cref{tab:ablation_nli}: (1) w/o PoP: a single best agent whose posterior is estimated by averaging calibrated NLI scores across \(N=5\) Monte Carlo samples (\cref{eq:approximate_posterior}), without the cross-agent product in \cref{eq:PoP}.
(2) w/o calibration: MA-PoP using softmax-normalised NLI logits in place of the calibrated posterior estimate. (3) Temperature scaling: MA-PoP in which the Deep Sets calibration head is replaced by temperature scaling~\citep{guo2017calibration,platt1999probabilistic}, which learns a single scalar temperature to rescale the logits via softmax. 
Overall, these results indicate that pooling and calibration are complementary, with pooling providing substantial additional gains and calibration being necessary to realise them.
The w/o calibration condition shows that multiplying uncalibrated scores yields only marginal gains, confirming that well-formed probabilities are a prerequisite for multiplicative aggregation. Temperature scaling improves performance on some benchmarks (e.g., MMLU Pro. Med., CSQA) but underperforms on others (HellaSwag, MMLU Form. Log.), and consistently falls short of the full MA-PoP. This suggests that a single scalar parameter is insufficient to capture the permutation-equivariant structure of multiple-choice options. These results indicate that calibration and posterior pooling are complementary, with pooling providing the larger contribution and calibration being necessary to realise it.
\begin{table}[th]
    \centering
    \caption{Ablation results for MA-PoP in the heterogeneous setting. Single-agent Best denotes the best results of single agent baseline. The results consist of the mean and standard deviations obtained from three different random seeds. }
    \resizebox{\linewidth}{!}{%
    \begin{tabular}{l ccccccc}
    \toprule
     & \textbf{\shortstack{MMLU\\(Pro.Med.)}} & \textbf{\shortstack{MMLU\\(Form.Log.)}} & \textbf{HellaSwag} & \textbf{CSQA} & \textbf{HH-RLHF} & \textbf{MedMCQA}\\
    \midrule
    \multicolumn{7}{c}{\textbf{Single-Agent}} \\
    \midrule
    Single-agent Best & $0.8015 \pm .01$ & $0.5873 \pm .01$ & $0.8267 \pm .01$ & $0.8300 \pm .01$ & $0.5567 \pm .01$ & $0.5767 \pm .01$\\
    \midrule
     \multicolumn{7}{c}{\textbf{Multi-Agent}} \\
    \midrule
    \textbf{MA-PoP (w/o calibration)}  & 0.8567 \(\pm\) .02  & 0.6387 \(\pm\) .02 & 0.8333 \(\pm\) .01 & 0.8500 \(\pm\) .01 & 0.5567 \(\pm\) .01 & 0.6100 \(\pm\) .01\\
    \textbf{MA-PoP (w/o PoP)}  & 0.8235 \(\pm\) .01 & 0.6847 \(\pm\) .02 & 0.8233 \(\pm\) .01 & 0.8167 \(\pm\) .02 & 0.5567 \(\pm\) .01 & 0.5933 \(\pm\) .01\\
    \midrule
    \textbf{MA-PoP (w. temperature)}  & 0.8713 \(\pm\) .01 & 0.6349 \(\pm\) .02 & 0.7900 \(\pm\) .02 & 0.8633 \(\pm\) .00 & 0.5567 \(\pm\) .01 & 0.6133 \(\pm\) .01\\
    \midrule
    \textbf{MA-PoP}  & \textbf{0.8787 $\pm$ .01}  & \textbf{0.7367 \(\pm\) .01}   & \textbf{0.8433 \(\pm\) .01}  & \textbf{0.8800 \(\pm\) .00}  & \textbf{0.5900 \(\pm\) .01}  & \textbf{0.6467 \(\pm\) .02}       \\
    \bottomrule
    \end{tabular}%
    }
    \label{tab:ablation_nli}
\end{table}

\clearpage
\section{Number of Monte Carlo samples}
\label{appendix:num_samples}
We analyse the impact of the number of Monte Carlo samples, denoted by \(N\), on the quality of posterior estimation in MA-PoP. The results in \cref{fig:numofsamples} for the homogeneous setting show substantial improvement when increasing from \(N = 1\) to \(N = 5\) samples, with performance stabilising beyond \(N \ge 5\).
\begin{figure}[th]
    \centering
    \begin{tikzpicture}
    \begin{axis}[
        height = 0.375\linewidth,
        width = 0.5\linewidth,
        xlabel={Number of Sample times \(N\)},
        ylabel={Accuracy},
        xlabel style={font=\footnotesize},
        ylabel style={font=\footnotesize, yshift=-0.375em},
        xticklabel style={font=\footnotesize},
        yticklabel style={font=\footnotesize},
        xtick distance={2},
        xmin=0.5,
        xmax=10.5,
        legend image post style={scale=0.75},
        legend cell align={left},
        legend pos=south west,
        legend columns=1,
        legend style={draw=none, text opacity=1, fill opacity=0.75, yshift=6ex, xshift=0.em, font=\footnotesize},
        scale only axis
    ]

        \addplot[mark=square*, mark options={scale=1., draw=MidnightBlue, fill=MidnightBlue, style=solid}, draw=MidnightBlue, style={dashed, thin}] coordinates {
            (1, 0.8456)
            (3, 0.8750)
            (5, 0.8800)
            (10, 0.8787)
        };
	   \addlegendentry{Pro. Medicine};

        \addplot[mark=pentagon*, mark options={scale=1.25, draw=Purple, fill=Purple, style=solid}, draw=Purple, style={dashdotdotted, thin}] coordinates {
            (1, 0.8167)
            (3, 0.8267)
            (5, 0.8433)
            (10, 0.8467)
        };
        \addlegendentry{Hellaswag};

        \addplot[mark=*, mark options={scale=1.25, draw=BurntOrange, fill=BurntOrange, style=solid}, draw=BurntOrange, style={solid, thin}] coordinates {
            (1, 0.6167)
            (3, 0.6133)
            (5, 0.6367)
            (10, 0.6333)
        };
	   \addlegendentry{MedMCQA};

    \end{axis}
\end{tikzpicture}
    \vspace{-1em}
    \caption{Comparison of different number of sample times \(M\) for homogeneous LLM agents.}
    \label{fig:numofsamples}
\end{figure}
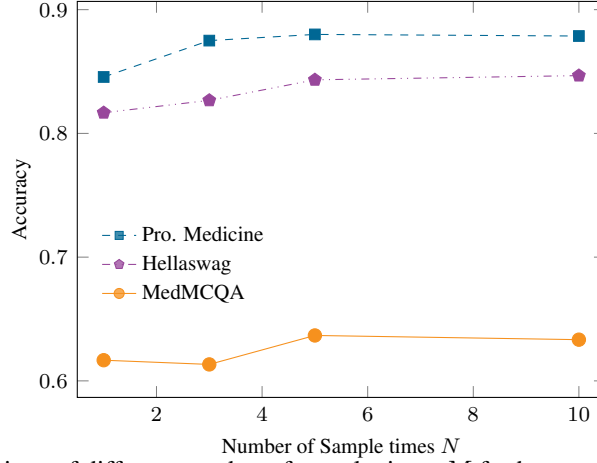

\newpage
\section{Computational Efficiency Analysis}
\label{appendix:computational_efficiency_analysis}

We evaluate the computational cost by measuring input, output, and the total number of tokens consumed per sample across all methods on the MedMCQA dataset in the heterogeneous setting of 5 agents. \cref{tab:token_comparison} presents token usage statistics accumulated across all debate rounds. MA-PoP demonstrates competitive efficiency compared to MAD approaches. The running time of NLI cross-encoder is less than 1 second per sample, which is negligible compared to the time needed for the whole inference, which is about 25 seconds.

\begin{table}[th]
\centering
\caption{Token usage per sample on MedMCQA with 5 heterogeneous agents. Token counts are cumulative across all rounds. }
\label{tab:token_comparison}
\resizebox{0.8\linewidth}{!}{
\begin{tabular}{l r r r}
    \toprule
     & \textbf{Input Tokens} & \textbf{Output Tokens} & \textbf{Total Tokens} \\
     \midrule
     \multicolumn{4}{c}{\textbf{Multi-Agent}} \\
    \midrule
    Decentr. MAD ($T=2$) & 5,621  & 1,381 & 7,002  \\
    Decentr. MAD ($T=3$) & 7,876  & 1,760 & 9,636  \\
    Decentr. MAD ($T=5$) & 12,248 & 2,529 & 14,777 \\
    Centr. MAD ($T=2$) & 6,250 & 2,283  & 8,532  \\
    Centr. MAD ($T=3$) & 8,885 & 3,002  & 11,887  \\
    Centr. MAD ($T=5$) & 14,286 & 4,536  & 18,822   \\
    Sparse MAD ($T=2$) & 4,807 & 1,579  & 6,386   \\
    Sparse MAD ($T=3$) & 6,996 & 2,134  & 9,129  \\
    Sparse MAD ($T=5$) & 11,450 & 3,137  & 14,587  \\
    Free MAD ($T=2$) & 10,105 & 2,282  & 12,387   \\
    Free MAD ($T=3$) & 14,571 & 2,973  & 17,544  \\
    Free MAD ($T=5$) & 18,692 & 3,403  & 23,095  \\
    \midrule
    \textbf{Majority Voting}  & 433 & 518  & 951   \\
    \textbf{Self-Consistency}  & 2,166 & 2,687  & 4,853   \\
    \textbf{LLP}  & 2,166 & 2,703  & 4,869   \\
    \textbf{ISP}  & 2,166       & 2,685        & 4,851   \\
    \midrule
    \textbf{MA-PoP}  & 2,166       & 2,696        & 4,862\\
    \bottomrule
    \end{tabular}%
    }
    
\end{table}

\clearpage
\section{Extend related work}
\label{sec:extended_related_work}

\paragraph{Multi-agent decision making} dates back to the 1970s, originating in \emph{decision support systems}, and has since evolved into complex multi-agent systems~\citep{keen1978decision}. The literature spans a diverse range of techniques, including rule-based systems~\citep{zhang2014leader, yan2023game}, game-theoretic methods~\citep{lanctot2017unified, wang2022cooperative}, reinforcement learning~\citep{sunehag2018value, vinyals2019grandmaster, son2019qtran, rashid2020monotonic}, and, more recently, LLM-based agents, each with distinct trade-offs. Rule-based methods remain interpretable but brittle and hard to scale or optimise. Game‑theoretic approaches offer theoretical guarantees at equilibrium, but their integration into real‑world pipelines is hindered by impractical modelling assumptions (e.g., requiring an infinite number of rounds to reach equilibrium). Reinforcement learning has delivered remarkable results, but training is data- and compute-intensive, and performance largely depends on reward specification. LLM‑based agents demonstrate strong performance due to their reasoning capabilities; however, current decision‑making aggregation remains ad‑hoc, typically relying on majority voting or debate and lacking formal guarantees on decision quality.
In contrast, we ground multi-agent decision making in Blackwell's informativeness framework to generate decisions with information-theoretic guarantees.

\paragraph{LLM uncertainty estimation} seeks to quantify an LLM agent's probabilistic belief in a given context. A straightforward approach is to elicit verbalised confidence through natural-language prompts that explicitly ask the model to report its confidence level~\citep{lin2022teaching, zheng2023ddcot, tian2023just, xiong2024can}. However, such self-reported confidence often fails to correspond to the model's true underlying belief state~\citep{kadavath2022language}. An alternative line of work employs sample-based estimation, where uncertainty is inferred from variability across multiple stochastic model outputs~\citep{lin2024generating}. The practical instantiation of our method, presented in \cref{sec:method}, follows this sample-based direction.

\paragraph{Product of Experts} (PoE)~\citep{hinton2002training} can be seen as an inspiration to our work, but it is fundamentally different from MA-PoP. In particular, PoE is primarily a modelling and training paradigm that combines multiple expert likelihoods without offering decision-theoretic optimality guarantees. 
In contrast, MA-PoP is derived from a decision-theoretic perspective: it approximates Bayesian posterior pooling over agents' private information and is Blackwell-optimal in the independent-private-information limit. In realistic LLM settings, where agents may be correlated, MA-PoP should instead be viewed as an approximation to posterior-level evidence aggregation.

\clearpage
\section{Limitations and Future Works}
\label{appendix:limitations}

Despite promising empirical results, our methodology has an intrinsic limitation: it relies on the assumption of conditional independence among agents' private observations or private training datasets (see \cref{eq:PoP}). In practice, this assumption implicitly requires sufficient diversity of agents' private information sources. When diversity is present, conditional independence serves as a tractable approximation that facilitates posterior factorisation and enables a practical estimator of the pooled posterior.

The conditional independence assumption is, however, unlikely to hold in practical LLM deployments. Agents' private knowledge, spanning pre-training data, fine-tuning procedures, and toolchains, typically overlaps. This leads to overconfidence when using PoP in \cref{eq:PoP} and, in turn, misleads downstream decision-makers. 
As future work to mitigate this limitation, we propose leveraging open-source and open-weight models to systematically characterise inter-agent dependencies and quantify their impact on pooled posterior.

\end{document}